\documentclass[review]{elsarticle}

\usepackage{microtype}
\usepackage{graphicx}
\usepackage{subfigure}
\usepackage{booktabs}
\usepackage{hyperref}

\usepackage{lineno,hyperref}
\modulolinenumbers[5]
\usepackage{amsmath}
\usepackage{amssymb}
\usepackage{amsthm}
\usepackage{amsfonts}
\usepackage{threeparttable}
\usepackage{algorithm} 
\usepackage{algorithmic}
\usepackage{color}
\newtheorem{definite}{Definition}
\newtheorem{thm}{Theorem}[section]
\newtheorem{lem}[thm]{Lemma}

\journal{Journal of \LaTeX\ Templates}

\begin{document}

\begin{frontmatter}

\title{Stable Sparse Subspace Embedding for Dimensionality Reduction}


\author[mymainaddress,mysecondaryaddress]{Li Chen}
\author[mysecondaryaddress]{Shuisheng Zhou\corref{mycorrespondingauthor}}
\cortext[mycorrespondingauthor]{Corresponding author}
\ead{sszhou@mail.xidian.edu.cn}
\author[mysecondaryaddress]{Jiajun Ma}
\address[mymainaddress]{College of Physical Education (Main Campus), Zhengzhou University, 100 Science Avenue, Zhengzhou, China}
\address[mysecondaryaddress]{School of Mathematics and Statistics, Xidian University, 266 Xinglong Section, Xifeng Road, Xi'an, China}

\begin{abstract}
Sparse random projection (RP) is a popular tool for dimensionality reduction that shows promising performance with low computational complexity. However, in the existing sparse RP matrices, the positions of non-zero entries are usually randomly selected. Although they adopt uniform sampling with replacement, due to large sampling variance, the number of non-zeros is uneven among rows of the projection matrix which is generated in one trial, and more data information may be lost after dimension reduction. To break this bottleneck, based on random sampling without replacement in statistics, this paper builds a stable sparse subspace embedded matrix (S-SSE), in which non-zeros are uniformly distributed. It is proved that the S-SSE is stabler than the existing matrix, and it can maintain Euclidean distance between points well after dimension reduction. Our empirical studies corroborate our theoretical findings and demonstrate that our approach can indeed achieve satisfactory performance.
\end{abstract}

\begin{keyword}
dimensionality reduction \sep feature projection \sep random projection \sep sparse \sep stable
\end{keyword}

\end{frontmatter}


\section{Introduction}
 Dimensionality reduction, which projects original features into a lower dimensional space, has been a prevalent technique in dealing with high dimensional datasets, because it is able to remove redundant features, reduce memory usage, avoid the curse of dimensionality and improve efficiency of machine learning algorithm. As a preprocessing step, dimensionality reduction has been applied to a variety of problems including $k$-means clustering \cite{Boutsidis2010, Sinha2018icml, Cai2017knowledge}, support vector machines classification \cite{shi2012margin, zhang2013recovering, kumar2008randomized, paul2014random}, $k$-nearest neighbors classification \cite{deegalla2006reducing}, least squares regression, and low rank approximation \cite{Kenneth2017}. However, how to design efficient and effective dimensionality reduction algorithm is a serious challenge problem.

  The goal of dimensionality reduction is to approximate a large matrix $X$ with a much smaller sketch $\hat X$ such that the solution to a given problem on $\hat X$ is a good approximation on $X$. Some works obtain $\hat X$ by low-rank approximation (also known as singular value decomposition (SVD) or principal component analysis(PCA)\cite{Michael2015}). Given a dataset $X\in\mathbb{R}^{m \times n}$, consisting of $m$ data points each having $n$ features, SVD requires $O(mn\min\{m,n\})$ time to reduce data dimensionality from $n$ to $d$ ($d\ll n$), which is prohibitively large even for moderate size datasets. By imposing sparse regularization, some sparse PCA based methods are proposed for dimension reduction, see \cite{Zou2006spca}\cite{Shen2008spca}\cite{leng2009aspca}.
 These low-rank approximation methods can preserve data information well, but they are all based on minimization optimization problems, so it is very hard to solve them and the computation is time consuming. To overcome this obstacle, we study random projection (RP) techniques in this article.


RP multiplies $X$ by the transpose of a random matrix $R\in\mathbb{R}^{d \times n}$, i.e. $\hat{X}=XR^\top \in \mathbb{R}^{m \times d}$, where $d$ is independent of $m$ and $n$, to satisfy $\|R\mathbf{x}\|_2\approx \|\mathbf{x}\|_2$ simultaneously for all samples $\mathbf{x}\in \mathbb{R}^n$ in $X$. 
%
It has been applied in various fields, such as image data \cite{Farhad2017}, text documents \cite{bingham2001random}, face recognition \cite{goel2005face}, privacy preserving distributed data mining \cite{liu2006random}, etc. Compared to SVD-based dimensionality reduction approaches, RP reduces the running time to at most $O(mnd)$. The critical factor affecting the efficiency and effectiveness of RP is the random matrix $R$. A good $R$ is able to make the process of dimensionality reduction efficient, and can well preserve the Euclidean distances between pairwise points after dimensionality reduction.

There are a number of literatures on designing $R$. In \cite{arriaga1999algorithmic}, the entries of $R$ (denoted by $R_{ij}$) obey standard normal distribution having mean $0$ and variance $1$, i.e. $R_{ij}\sim N(0,1)$. Achlioptas \cite{Achlioptas2001} demonstrates that $R_{ij}$ can also have values $+1$ or $-1$ with probability $1/2$, which we denote as $U(1,-1)$. It is proved that $R_{ij}$ in this method have mean $0$ and variance $1$, and the distribution of $R_{ij}$ is symmetric about the origin with $\mathbb{E}(R_{ij}^2) = 1$. This property is sufficient to prove that $(1+\epsilon)$-approximate holds after dimensionality reduction \cite{Achlioptas2001, arriaga1999algorithmic}. Comparing to $R_{ij}\sim N(0,1)$, the advantage of $R_{ij}\sim U(1,-1)$ is that the computation of the projection only contains summations and subtractions, but no multiplications, hence the computation is simple.
 However, because random matrices $R$ are both dense in these two methods, the computational complexity of multiplication $XR^\top$ are both $O(nnz(X)d)$, where $nnz(X)$ denotes the number of nonzero entries in $X$, and $nnz(X)=mn$ when $X$ is dense. This complexity is lower than SVD-based dimensionality reduction approaches as $d\ll n <m$ normally, but it is still high.

\begin{table*}[t]
  \centering
  \caption{Summary of RP methods. The second column corresponds to the type of the matrix $R$. The third column corresponds to the number of extracted features. The forth column corresponds to the number of nonzero entries per column in $R$. The fifth column corresponds to the time complexity of multiplication $XR^\top$. Approximate error are all $1+\epsilon$. $nnz(X)$ denotes the number of non-zeros in $X$. $\epsilon$ and $\delta$ represent the relative error of Euclidean distance and confidence level, respectively.}
  \vskip 0.15in
  \begin{small}
  \begin{threeparttable}
  \begin{tabular}{|c|c|c|c|c|}
  \hline
  Method&Type&Dimensions&\#nonzeros per column&Time for $XR^\top$\\
  \hline
  \cite{dasgupta1999learning}&Density&$O(\frac{\log m}{\epsilon^{2}})$&$O(\frac{\log m}{\epsilon^{2}})$&$O(\frac{nnz(X) \log m}{\epsilon^{2}})$\\
  \hline
  \cite{Achlioptas2001}&Density&$O(\frac{\log m}{\epsilon^{2}})$&$O(\frac{\log m}{\epsilon^{2}})$&$O(\frac{nnz(X)\log m}{\epsilon^{2}})$\\
  \hline
  \cite{Achlioptas2001}&Sparse&$O(\frac{\log m}{\epsilon^{2}})$&$O(\frac{\log m}{3\epsilon^{2}})$&$O(\frac{nnz(X)\log m}{3\epsilon^{2}})$\\
  \hline
  \cite{Daniel2014}&Sparse&$O(\frac{\log (1/\delta)}{\epsilon^{2}})$&$O(\frac{\log(1/\delta)}{\epsilon})$&$O(\frac{nnz(X)\log(1/\delta)}{\epsilon})$\\
  \hline
  This Paper&Sparse&$O(\frac{\log (1/\delta)}{\epsilon^{2}})$&1&$O(nnz(X)$\\
  \hline
\end{tabular}
\end{threeparttable}
\end{small}
\label{tab:diff_methods}
\vskip -0.1in
\end{table*}

 To further reduce the complexity of RP, researchers turn their attention to sparse matrices. The complexity of the multiplication $XR^\top$ is $O(nnz(X)\varrho)$ when $R$ is a sparse matrix, where $\varrho<d$ is the number of nonzero entries in per row. The smaller $\varrho$ is, the less computational cost of RP is. In \cite{Achlioptas2001} and \cite{Li2006}, $R_{ij}\in \{+\sqrt{\kappa}, -\sqrt{\kappa}\}$ with probability $1/2\kappa$, otherwise $0$, where $\kappa\geq3$ such as $\kappa=\sqrt{n}$ or $\kappa=n/\log{n}$. In each row of this matrix, $d/\kappa$ entries are non-zeros, where $d=O(\epsilon^{-2}\log m)$. In \cite{Daniel2014}, $R_{ij}=\eta_{ij}\sigma_{ij}/\sqrt{\varrho}$, where $\sigma_{ij}$ are independent and uniform in $\{-1,+1\}$, $\eta_{ij}$ are indicator random variables for $R_{ij}\neq 0$. Each column of this matrix exactly has $\varrho\geq 2(2\epsilon-\epsilon^2)^{-1}log(1/\delta)>1$ nonzero entries, where $0< \epsilon, \delta <1/2$. These methods are all able to get $(1+\epsilon)$-approximation of Euclidean distance
between points. However, $\varrho$ are all larger than 1, thus $R$ are not sufficiently sparse. Recently, Clarkson et al. \cite{Kenneth2017} and Liu et al. \cite{liu2017sparse} constructed a very sparse embedded (SE) matrix $R$ with $R_{ij}\in\{+1,-1,0\}$. In $R$, each column only contains one nonzero entry. The computational complexity of the multiplication $XR^\top$ is only $O(nnz(X))$, which is the lowest as far as we know. Table \ref{tab:diff_methods} summarizes the properties of the above mentioned methods.

There is one defect in the existing RP matrices that the positions of nonzero entries in each column of $R$ are random. Although the row labels of non-zero entries in each column are obtained by uniform sampling with replacement from $\{1,\ldots,d\}$, such sampling manner leads to a large variance, therefore the number of non-zeros is uneven among rows of the RP matrix that is generated in one trial, which may cause more data information loss after dimension reduction and leads to bad Euclidean distance preservation between points. Moreover, the large variance also causes the generated RP matrices instability, and further leads to the performance of dimension reduction unstable.



To improve stability of the sparse RP matrices as well as reduce variance of the number of nonzero entries among rows in matrix, we use the ideas of randomly sampling without replacement in statistics. To the best of our knowledge, this is the first attempt to improve the stability of RP matrices, and our method is simple and effective. The main contributions are summarized as follows.
\begin{itemize}
  \item The stable sparse subspace embedded matrix is constructed for dimension reduction. In this construction, the idea of uniform sampling without replacement is adopted to obtain the position of nonzero entries in the matrix. In the constructed matrix, each row contains $\lfloor\frac{n}{d}\rfloor$ or $\lfloor\frac{n}{d}\rfloor+1$ nonzero entries, and each column contains only one nonzero. 
  \item We prove that our matrix is stabler than SE matrix \cite{liu2017sparse}.
  \item It is proved that embedding the original data into dimension $d=O(\epsilon^{-2}\log (1/\delta))$ is sufficient to preserve all the pairwise Euclidean distances up to $1\pm\epsilon$. 
  \item Experimental results verify our theoretical analysis, and illustrate that our algorithm outperforms other compared dimension reduction methods.
\end{itemize}

The rest of this paper is organized as follows. Section 2 gives notations used in this paper and introduces theoretical basis of random projections. Section 3 describes sparse embedding method. We propose our stable sparse subspace embedding in section 4 and present its analysis in section 5. Experimental results are presented in section 6. Finally, we summarize the whole article and point out a few questions in section 7.
\section{Preliminaries}
\subsection{Notations and linear algebra}
$X\in \mathbb{R}^{m\times n}$ is the dataset with $m$ samples and $n$ features.
We denote $d$ as the number of reduced features. All logarithms are base-$2$ by $\log$. For a positive integer $n$, we use $[n]$ to denote the set $[1,\ldots,n]$. $\lceil\cdot\rceil$ denotes the smallest integer greater than a number, and $\lfloor\cdot\rfloor$ denotes the largest integer less than a number. $\mathbb{P}(\cdot)$ is the probability of an event. A vector $\mathbf{x}$ is assumed to be a row vector, and $\mathbf{x}^\top$ denotes its transpose. For a vector $\mathbf{x}\in \mathbb{R}^n$, $\|\mathbf{x}\|_2=\sqrt{\sum_{i=1}^n x_i^2}$. For a matrix $R\in \mathbb{R}^{d\times n}$, $\|R\|_F=\sqrt{\sum_{i,j} R_{ij}^2}$ and $\|R\|_2=\sup_{\|\mathbf{x}\|_2=1}\|R\mathbf{x}\|_2=\sqrt{\lambda_{max}(R^\top R)}$, i.e. the square root of the largest eigenvalue of $R^\top R$. $R_{i \cdot }$ denotes all the entries of the $i$-th row in $R$.
\subsection{Theoretical basis of random projections}
RP is a computationally efficient and sufficiently accuracy method as respect to preserving Euclidean distance after dimension reduction. The theoretical basis of RP arises from the following lemma:
\begin{lem}\label{lem:JL}
 \textbf{(Johnson-Lindenstrauss Lemma \cite{Johnson1984, Daniel2014})} For any real numbers $0<\epsilon, \delta<1/2$, there exists an absolute constant $C>0$, such that for any integer $d=C\epsilon^{-2}\log(1/\delta)$, there exists a probability distribution $\mathcal{D}$ on $d\times n$ real matrices such that for any fixed $\mathbf{x}\in \mathbb{R}^n$,
   \begin{equation*}
   \mathbb{P}_{R\sim \mathcal{D}}((1-\epsilon)\|\mathbf{x}\|_2\leq \|R\mathbf{x}\|_2\leq(1+\epsilon)\|\mathbf{x}\|_2)>1-\delta.
   \end{equation*}
   where $R\sim \mathcal{D}$ indicates that the matrix $R$ is a random matrix with distribution $\mathcal{D}$. $\mathbb{P}$ is the probability of a event.
\end{lem}
Using linearity of $R$ and Lemma \ref{lem:JL} with $\mathbf{x}=\mathbf{u}-\mathbf{v}$, we get that $R$ satisfies $(1-\epsilon)\|\mathbf{u}-\mathbf{v}\|_2 \leq\|R\mathbf{u}-R\mathbf{v}\|_2 \leq (1+\epsilon)\|\mathbf{u}-\mathbf{v}\|_2$ with probability at least $1-\delta$. Therefore, Johnson-Lindenstrauss lemma illustrates that if points in one space are projected onto a randomly extracted subspace with suitable dimension, then the distance between pairwise points are approximately preserved \cite{bingham2001random}. In order to satisfy Johnson-Lindenstrauss Lemma, the entries of random projection matrix $R$ should be i.i.d. with zero mean and unit variance \cite{arriaga1999algorithmic, Achlioptas2001, Daniel2014}. For convenience, we define subspace embedded matrix as follows.
\begin{definite}\label{def:embed}
(Subspace embedded matrix) Given $0< \epsilon, \delta <1$, matrix $R\in \mathbb{R}^{d\times n}$ is a subspace embedded matrix, if for any $\mathbf{x}\in \mathbb{R}^n$,
  \begin{equation*}
    \mathbb{P}((1-\epsilon)\|\mathbf{x}\|_2\leq \|R\mathbf{x}\|_2\leq(1+\epsilon)\|\mathbf{x}\|_2)>1-\delta.
  \end{equation*}
  Moreover, if matrix $R$ is a sparse matrix, then $R$ is a sparse subspace embedded matrix. The probability $\mathbb{P}((1-\epsilon)\|\mathbf{x}\|_2\leq \|R\mathbf{x}\|_2\leq(1+\epsilon)\|\mathbf{x}\|_2)$ is called distance preservation probability.
\end{definite}
The Definition \ref{def:embed} indicates that matrix $R$ embeds space $\mathbb{R}^n$ into $\mathbb{R}^{d}$ while preserving the distance between points $(1+ \epsilon)$-approximation with the probability larger than $1-\delta$. A good subspace embedded matrix makes the Euclidean distance approximation better, and calculates multiplication $XR^\top$ fast.
\section{Sparse embedding}\label{sec:SE}
The sparse embedding algorithm is listed in Algorithm \ref{alg:SE}.

\begin{algorithm}[ht]
\renewcommand{\algorithmicrequire}{\textbf{Input:}}
\renewcommand\algorithmicensure {\textbf{Output:} }
\caption{\textbf{Sparse Embedding \cite{liu2017sparse}}}
\begin{algorithmic}[1]\label{alg:SE}
\REQUIRE Dataset $X\in\mathbb{R}^{m\times n}$.
\ENSURE Sparse embedded matrix $R=\Phi Q\in\mathbb{R}^{d\times n}$ and feature extracted matrix $\hat{X}\in \mathbb{R}^{m\times d}$.
\STATE Build a random map $h$ so that for any $i\in[n]$, $h(i)=j$ for $j\in[d]$ with probability $1/d$, where $0<d<n$.
\STATE Construct matrix $\Phi\in\{0,1\}^{d\times n}$ with $\Phi_{h(i),i}=1$, and all remaining entries 0.
\STATE Construct matrix $Q\in \mathbb{R}^{n\times n}$ is a random diagonal matrix whose entries are i.i.d. Rademacher variables.
\STATE Compute the product $\hat{X}=X(\Phi Q)\top$.
\end{algorithmic}
\end{algorithm}

In Algorithm \ref{alg:SE}, $h$ is a random map so that the row labels of the nonzero entries in $\Phi$ are completely random. This causes that the distribution of nonzero entries is uneven between rows, that is, some rows in $R$ contain more nonzero entries but other rows contain less even none, see Fig.\ref{fig:matrix_a} for an example. In Fig.\ref{fig:matrix_a}, the fifth row contains 10 nonzeros. But the eighth row does not contain any nonzeros. For feature extraction $XR^\top$, the fifth row in the SE matrix indicates that ten features of $X$ are linear combined into one feature, which may lead to more information loss. Moreover, the randomness of position of nonzero entry in per column of $R$ results in $R$ instability, because it is equivalent to random sampling from $[d]$ with replacement as the row label of nonzero entry in per column, the variance of which is large.
\begin{figure}[ht]
\vskip 0.2in
\centering
 \subfigure[SE matrix]{
    \label{fig:matrix_a} 
    \includegraphics[width=\textwidth]{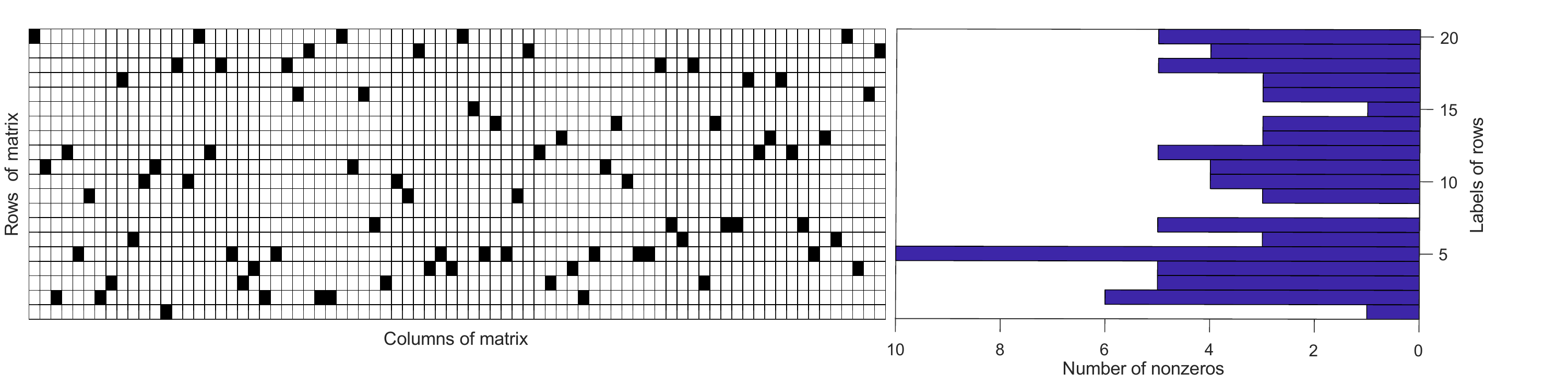}} 
 \subfigure[S-SSE matrix]{
    \label{fig:matrix_b} 
    \includegraphics[width=\textwidth]{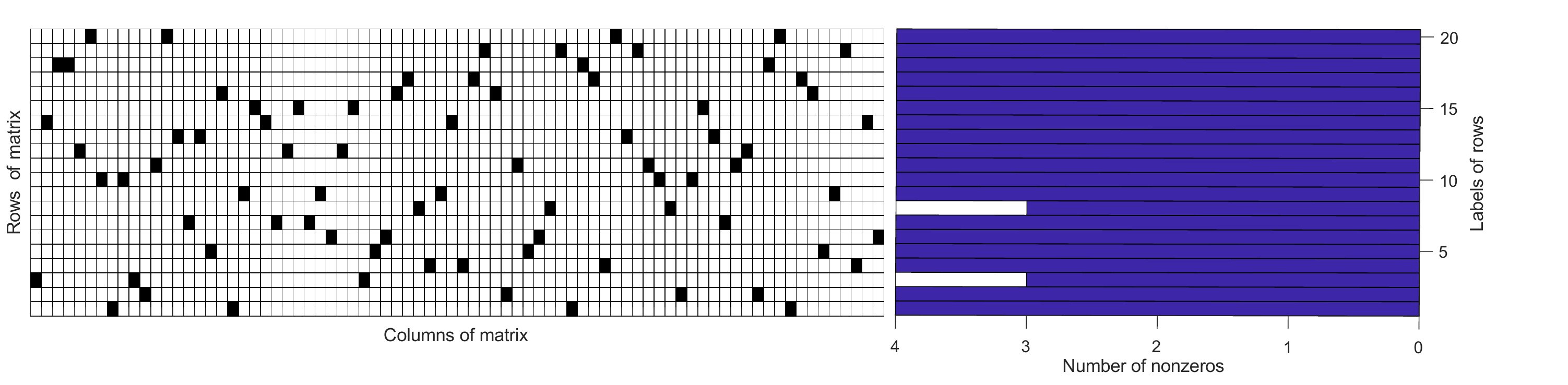}}
 \caption{Nonzero entries in two sparse matrices $R\in\mathbb{R}^{20\times78}$. The black and white boxes in left figures denote the nonzero and zero entries, respectively. Each column only contains one nonzero. Right figures are the number of nonzero entries in each row of matrices. In the Fig.\ref{fig:matrix_a}, the distribution of nonzeros is uneven. Some rows of $R$ contain more nonzero entries but other rows contain less even none. In the Fig.\ref{fig:matrix_b}, the distribution of nonzeros is even. Each row contains 4 or 3 nonzeros.}
 \label{fig:matrix} 
 \vskip -0.2in
 \end{figure}

In the following sections, we build a new sparse subspace embedding matrix and provide theoretical analysis for it in order to overcome the defects of SE.
\section{Stable sparse subspace embedding}\label{sec:S-SSE}
In this section, we design a new sparse subspace embedded (SSE) matrix: Stable SSE matrix (S-SSE). Algorithm \ref{alg:1} gives the construction of S-SSE matrix. In this matrix, each column only has one nonzero entry, which is $+1$ or $-1$ with the same probability. Every row contains almost the same number of non-zeros. 

 \begin{algorithm}[htp]
\renewcommand{\algorithmicrequire}{\textbf{Input:}}
\renewcommand\algorithmicensure {\textbf{Output:} }
\linespread{1.1}\selectfont
\caption{\textbf{Stable Sparse Subspace Embedding (S-SSE)}}
\begin{algorithmic}[1]\label{alg:1}
\REQUIRE Dataset $X\in\mathbb{R}^{m\times n}$.
\ENSURE Embedded matrix $R\in\mathbb{R}^{d\times n}$ and feature extracted matrix $\hat{X}$.
\STATE Set $d=O(\frac{\log (1/\delta)}{\epsilon^2})$.
\STATE Repeat $[d]$ for $\lceil n/d \rceil$ times and obtain a set $D$.
\STATE Randomly sample $n$ elements from $D$ without replacement to construct sequence $\mathcal{S}$.
\STATE Construct matrix ${R}\in\{0,+1,-1\}^{d\times n}$, where ${R}_{\mathcal{S}(i),i}\in\{+1,-1\}$ for $i\in[n]$ with probability $1/2$, and all remaining entries $0$.
\STATE Compute the multiplication $\hat{X}=XR^\top$.
\end{algorithmic}
\vspace*{-3pt}
\end{algorithm}

\emph{\textbf{Remark 1}}. Main difference between S-SSE and SE is the selection of row labels of nonzero entries. SE chooses those by randomly sampling with replacement, whereas our method chooses those by randomly sampling without replacement. The number of nonzero entries in each row of S-SSE matrix is $\lfloor \frac{n}{d} \rfloor$ or $\lceil \frac{n}{d} \rceil$, thus nonzero entries are uniformly distributed among columns of $R$, see Fig.\ref{fig:matrix_b} for an example. Furthermore, because the sampling error of sampling without replacement is smaller than that of sampling with replacement, $R$ constructed by the S-SSE follows a symmetric distribution about zero mean with unit variance better than by the SE, and the S-SSE satisfies the Johnson-Lindenstrauss lemma better \cite{arriaga1999algorithmic, Achlioptas2001, Daniel2014}. This leads to the S-SSE preserving the Euclidean distance better than the SE after dimension reduction, see the experimental results in Figures \ref{fig:distance} - \ref{fig:emsrong-delta}. Therefore, S-SSE may reduce data information loss after feature extraction comparing to the SE as distance between points is the important data information. Moreover, Section \ref{sec:stability} demonstrates that the matrix constructed by our method is stabler than by the SE. 

\emph{\textbf{Remark 2}}. The feature extraction is simple by using the S-SSE. It just needs to add or subtract original features in $X$ to form a new feature, i.e. linear combination of features in $X$ corresponding to the column labels of nonzero entries in the row of $R$. The computation complexity of feature extraction is also only $O(nnz(X))$, which is the same as SE method.
\section{Properties of the S-SSE}\label{sec:properties}
In this section, we prove two good properties of the S-SSE: stability of matrix and preservation of Euclidean distances.
\subsection{Stability of matrix}\label{sec:stability}
The following discussion confirms that the S-SSE matrix is stabler than the SE matrix.

The SE and the S-SSE matrices both contain only one nonzero entry in each column. Therefore, the stability of matrices is determined by the change of non-zero entries in rows. We employ the variance of the number of nonzeros in rows to measure the stability of a matrix. Denote the number of nonzeros in rows of the S-SSE matrix as $\mathbf Y$, then the possible values of $\mathbf Y$ are $\lceil \frac{n}{d} \rceil$ or $\lfloor \frac{n}{d} \rfloor$. Denote the number of nonzeros in rows of the SE matrix as $\mathbf Z$. The possible values of $\mathbf Z$ are $0,1,\ldots,n$, because each row of the SE matrix contains $n$ entries, and the position of non-zero entry in each column is selected randomly. Theorem \ref{thm:stable} indicates that the expectation of $\mathbf Y$ is the same as that of $\mathbf Z$, while the variance of $\mathbf Y$ is less than that of $\mathbf Z$ when $d\geq 2$.

\begin{thm}\label{thm:stable}
Denote $\mathbb{E}(\cdot)$ and $Var(\cdot)$ as the expectation and variance of a variable, respectively. The random variables $\mathbf Y$ and $\mathbf Z$ are the number of nonzeros in rows of the S-SSE matrix and the SE matrix, respectively, then $$\mathbb{E}(\mathbf Y)=\mathbb{E}(\mathbf Z), ~~ Var(\mathbf Y)\leq Var(\mathbf Z).$$
\end{thm}

\begin{proof}
Set $n=rd+q$, where $r=\lfloor\frac{n}{d}\rfloor$ and $0\leq q< d$ is an integer, the distribution of $\mathbf Y$ is
$$\mathbb{P}(\mathbf{Y}=r)=1-\frac{q}{d}, ~\mathbb{P}(\mathbf{Y}=r+1)=\frac{q}{d}.$$
The expectation of $\mathbf{Y}$ is
\begin{equation}\label{eq:qiwang_Y}
\mathbb{E}(\mathbf{Y})=r(1-\frac{q}{d})+(r+1) \frac{q}{d}=\frac{n}{d}.
\end{equation}
In addition, because
\begin{equation*}
\mathbb{E}(\mathbf{Y}^2)=r^2(1-\frac{q}{d})+(r+1)^2 \frac{q}{d},
\end{equation*}
the variance of $\mathbf{Y}$ is
\begin{equation}\label{eq:fangcha_Y}
Var(\mathbf{Y})=\mathbb{E}(\mathbf{Y}^2)-[\mathbb{E}(\mathbf{Y})]^2=\frac{q}{d}-(\frac{q}{d})^2.
\end{equation}

In the following, we compute the expectation and variance of $\mathbf{Z}$. Let random event $B$ mean ``non-zero is in the $i$-th row" and $\bar{B}$ mean ``non-zero is not in the $i$-th row". Because the row label of non-zero entry in each column is randomly chosen, which is equivalent to randomly sampling with replacement from $[d]$, therefore $\mathbb{P}(B)=\frac{1}{d}$ and $\mathbb{P}(\bar{B})=1-\frac{1}{d}$. The random variable $\mathbf{Z}$ is the number of times that $B$ occurs in $n$ Bernoulli trials. Hence $\mathbf{Z}$ obeys the binomial distribution, and the distribution of $\mathbf{Z}$ is
$$\mathbb{P}(\mathbf{Z}=k)=C_n^k(\frac{1}{d})^k(1-\frac{1}{d})^{n-k},~k=0,1,\ldots,n.$$
The expectation and variance of $\mathbf{Z}$ are
\begin{equation}\label{eq:qiwang_Z}
\mathbb{E}(\mathbf{Z})=\frac{n}{d},
\end{equation}
\begin{equation}\label{eq:fangcha_Z}
Var(\mathbf{Z})=n(\frac{1}{d})(1-\frac{1}{d}).
\end{equation}

Eqs. \eqref{eq:qiwang_Y} and \eqref{eq:qiwang_Z} indicate that $\mathbf{E}(\mathbf{Y})=\mathbf{E}(\mathbf{Z})$. Next, we prove $Var(\mathbf{Y})\leq Var(\mathbf{Z})$. If $d=1$, then $Var(\mathbf{Y})= Var(\mathbf{Z})=0$. If $2\leq d\leq n$, then $Var(\mathbf{Z})\geq \frac{n-1}{n}\geq \frac{1}{2}$, while $Var(\mathbf{Y})\leq \frac{1}{4}$, hence $Var(\mathbf{Y})< Var(\mathbf{Z})$. 
Therefore, $Var(\mathbf{Y})\leq Var(\mathbf{Z})$, where the equality sign holds only when $d=1$.
\end{proof}

\emph{\textbf{Remark 3}}. Eq. \eqref{eq:fangcha_Y} indicates that the variance of $\mathbf Y$ is related to $q=n ~ \text{mod} ~d$. When $q=0$, then $Var(\mathbf Y)=0$, that is, if $n$ can be divided by $d$ without remainder, then each row of the S-SSE matrix contains the same number of non-zeros. When $q=\frac{d}{2}$, the $Var(\mathbf{Y})$ reaches the maximum $\frac{1}{4}$. In comparison, the $Var(\mathbf{Z})$ is not less than $\frac{1}{4}$.

\emph{\textbf{Remark 4}}. Theorem \ref{thm:stable} illustrates that the number of non-zeros in rows of the SE matrix changes greater than that of the S-SSE matrix, which leads to large variety among rows in the SE matrix, and further causes the generated matrices changes greatly. Therefore, the SE matrix is more unstable than S-SSE matrix.

\subsection{Preservation the Euclidean distances}\label{sec:preservation}
 In this subsection, we prove that our S-SSE matrix can preserve pairwise Euclidean distance up to $1\pm \epsilon$.

\begin{lem}\label{lem:Hanson-Wright inequality}
\cite{Diakonikolas2010, Daniel2014} Let $B \in \Re^{n\times n}$ be symmetric and $\mathbf{z} \in \{+1, -1\}^n$ be random. Then for all $l \geq 2$,
$$\mathbb{E}[|(\mathbf{z}^\top B \mathbf{z}) - tr(B)|^l] \leq C^l \cdot \max\{\sqrt{l}\|B\|_F, l\|B\|_2\}^l$$
where $C>0$ is a universal constant.
\end{lem}

 \begin{thm}\label{thm:preserve}
 The matrix $R\in \mathbb{R}^{d\times n}$ is constructed by Algorithm \ref{alg:1}. Given $0<\epsilon,\delta<\frac{1}{2}$, there exists $d=O(\frac{\log(1/\delta)}{\epsilon^2})$ such that $R$ is a sparse subspace embedding matrix, i.e. for any $\mathbf{x} \in \mathbb{R}^n$,
 \begin{equation}\label{eq:thm1}
 \mathbb{P}((1-\epsilon)\|\mathbf{x}\|_2\leq \|R\mathbf{x}\|_2\leq(1+\epsilon)\|\mathbf{x}\|_2)>1-\delta.
 \end{equation}
 \end{thm}
 \begin{proof}
 Assume $\mathbf{x}$ is a unit vector, i.e. $\|\mathbf{x}\|_2^2=1$, which can be obtained in data preprocessing step. Therefore, \eqref{eq:thm1} is translated into
 $$\mathbb{P}(1-\epsilon\leq \|R\mathbf{x}\|_2\leq 1+\epsilon)>1-\delta.$$
 It is equal to the following inequation:
 $$\mathbb{P}(|\|R\mathbf{x}\|_2^2-1|>2\epsilon-\epsilon^2)<\delta.$$
 For convenience, we denote $h=\|R\mathbf{x}\|_2^2-1$, then \eqref{eq:thm1} is equal to
 \begin{equation}\label{eq:thm2}
 \mathbb{P}(|h|>2\epsilon-\epsilon^2)<\delta.
 \end{equation}

 We rewrite the entries of matrix $R$ as $R_{ij}=\eta_{ij}\sigma_{ij}$, where $\eta_{ij}$ is an indicator random variable for $R_{ij}\neq 0$, $\sigma_{ij}\in \{+1,-1\}$, 
 then
 \begin{equation*}
 h=\|R\mathbf{x}\|_2^2-1=\sum\limits_{t=1}^{d}\sum\limits_{i\neq j\in[n]}\eta_{ti}\eta_{tj}\sigma_{ti}\sigma_{tj}x_ix_j:=\sigma^\top A \sigma,
 \end{equation*}
 where $A$ is a $dn\times dn$ block diagonal matrix. It can be divided into $d$ blocks with each $n\times n$. For the $t$-th block $A_t$,
 \begin{equation*}
  (A_t)_{ij}=
  \begin{cases}
  \eta_{ti}\eta_{tj}x_ix_j, & i\neq j,\\
  0,&i=j.
  \end{cases}
 \end{equation*}
 Then,
 \begin{equation}\label{eq:thm3}
 \begin{split}
 \mathbb{P}(|h|>2\epsilon-\epsilon^2)&=\mathbb{P}(|\sigma^\top A \sigma|>2\epsilon-\epsilon^2)\\
 &=\mathbb{P}(|\sigma^\top A \sigma-\mathrm{tr}(A)|>2\epsilon-\epsilon^2)\\
 &=\mathbb{P}(|\sigma^\top A \sigma-\mathrm{tr}(A)|^l>(2\epsilon-\epsilon^2)^l)\\
 &\leq (2\epsilon-\epsilon^2)^{-l}\mathbb{E}(|\sigma^\top A \sigma-\mathrm{tr}(A)|^l)\\
 &\leq (2\epsilon-\epsilon^2)^{-l} C^l \max\{ \sqrt{l}\|A\|_F, l\|A\|_2\}^l
 \end{split}
 \end{equation}
 where $\mathrm{tr}(A)$ is the trace of the matrix $A$ and $\mathrm{tr}(A)=0$ as $A_{ii}=0$. $l\geq 2$. $C>0$ is some universal constant. The first inequality uses Markov-Bound. The second inequality uses Lemma \ref{lem:Hanson-Wright inequality} with $\mathbf{z}=\sigma$ and $B=A$. Next, we compute the bounds of $\|A\|_F$ and $\|A\|_2$.

 For any $i\neq j\in [n]$, $\sum_{t=1}^{d}\eta_{ti}\eta_{tj}\leq 1$, which indicates that the number of non-zero entries in the same row is no more than $1$ in two columns. We have
 \begin{equation}\label{eq:thm_Fnorm}
 \begin{split}
 \|A\|_F^2&=\sum\limits_{i\neq j\in [n]}x_i^2x_j^2\sum_{t=1}^{d}\eta_{ti}\eta_{tj}\\
 &\leq \sum\limits_{i\neq j\in [n]}x_i^2x_j^2 \leq \|\mathbf{x}\|_2^4 \leq 1.
 \end{split}
 \end{equation}

 Moreover, we can prove that
 \begin{equation}\label{eq:thm_2norm}
 \|A\|_2\leq 1.
 \end{equation}

Rewrite $A_t$ as $A_t=\bar{R}_t-\bar{D}_t$, here $(\bar{R}_t)_{ij}=\eta_{ti}\eta_{tj}x_ix_j$, $\bar{D}_t$ is a diagonal matrix with $(\bar{D}_t)_{ii}=\eta_{ti}x_i^2$. Because $\bar{R}_t$ and $\bar{D}_t$ are both positive semidefinite, we have $\|A\|_2\leq \max\{\|\bar{R}_t\|_2, \|\bar{D}_t\|_2\}$. $\|\bar{D}_t\|_2\leq \|\mathbf{x}\|_\infty^2\leq 1$. Denote $v_i=\eta_{ti}x_i$ and $\mathbf{v}\in \mathbb{R}^n$, then $\bar{R}_t=\mathbf{v}\mathbf{v}^\top$ and $\|\bar{R}_t\|^2=\|\mathbf{v}\mathbf{v}^\top\|_2^2\leq \|\mathbf{x}\|_2^2=1$. Therefore, $\|A\|_2\leq 1$.

 Substitute \eqref{eq:thm_Fnorm} and \eqref{eq:thm_2norm} into \eqref{eq:thm3}, we obtain
 \begin{equation}\label{eq:delta}
 \mathbf{P}(|h|>2\epsilon-\epsilon^2)\leq (2\epsilon-\epsilon^2)^{-l}C^ll^l<(\frac{1}{3}\cdot \frac{Cl}{\epsilon^2})^l.
 \end{equation}
  Let $C=\frac{C_1}{d}$, where $C_1>0$ is a constant, $l=\log(1/\delta)$. In order to make \eqref{eq:delta} less than $\delta=(\frac{1}{2})^l$, we need $d>\frac{2}{3}C_1\frac{l}{\epsilon^2}=O(\frac{\log(1/\delta)}{\epsilon^2})$.
 Therefore, the theorem \ref{thm:preserve} is proved.
 \end{proof}


 \emph{\textbf{Remark 5}}. With regard to the SE method, Clarkson et al. proved that $d=O((u/\epsilon)^4\log^2(u/\epsilon))$ can make $\|R\mathbf{x}\|_2=\|\mathbf{x}\|_2$ with probability at least $9/10$ \cite{Kenneth2017}, where $u$ is the rank of $X$. Liu et al. proved that $d=O(\max\{\frac{k+\log(1/\delta)}{\epsilon^2}, \frac{6}{\epsilon^2\delta}\})$ can get the $\epsilon$-approximately optimizing solution of $k$-means clustering \cite{liu2017sparse}. By comparison, we demonstrate that $d=O(\frac{\log(1/\delta)}{\epsilon^2})$ is sufficient for S-SSE to preserve Euclidean distance up to $(1+ \epsilon)$-approximation, and our proof is simpler.
 \section{Experiment}
 We compare our method S-SSE with several other feature extraction methods to evaluate the performance of the S-SSE. They are listed below:
\begin{itemize}
  \item\textbf{SPCA}: Sparse principal component analysis is proposed by \cite{Zou2006spca}. SPCA imposes the lasso (elastic net) constraint into the PCA to promote sparse. The matrix deduced by SPCA is a sparse matrix.
  \item \textbf{DE}: The density embedding (DE) method is proposed by \cite{Achlioptas2001}. In this method, $R$ is dense, $R_{ij}\in \{1,-1\}$ with the same probability.
  \item \textbf{SE}: The sparse embedding (SE) method corresponds to Algorithm \ref{alg:SE}. In this method, the position of nonzero entry in each column is randomly chosen.
  \item \textbf{S-SSE}: Stable sparse subspace embedding (S-SSE) corresponds to Algorithm \ref{alg:1}.
\end{itemize}

We performed all the experiments on the PC machine with dual Intel core i7-4790 CPUs at 3.60GHz processor and 8 GB of RAM.
 \subsection{Data separability comparison after dimensionality reduction}
In order to verify our theoretical analysis in section \ref{sec:stability}, we performed experiments on a synthetic dataset which consists of four classes. Each class contained 1000 samples with a dimension of 100. Features in four classes were drawn from normal distribution having variance 0.5 and mean 0, 2, 4 and 6, respectively. The dimension was reduced by using the SE and the S-SSE. Figure \ref{fig:4lei_data} shows data distribution when each class containing 100 samples with a dimension of 2.

\begin{figure}[!h]
\centering
\includegraphics[width=0.5\textwidth]{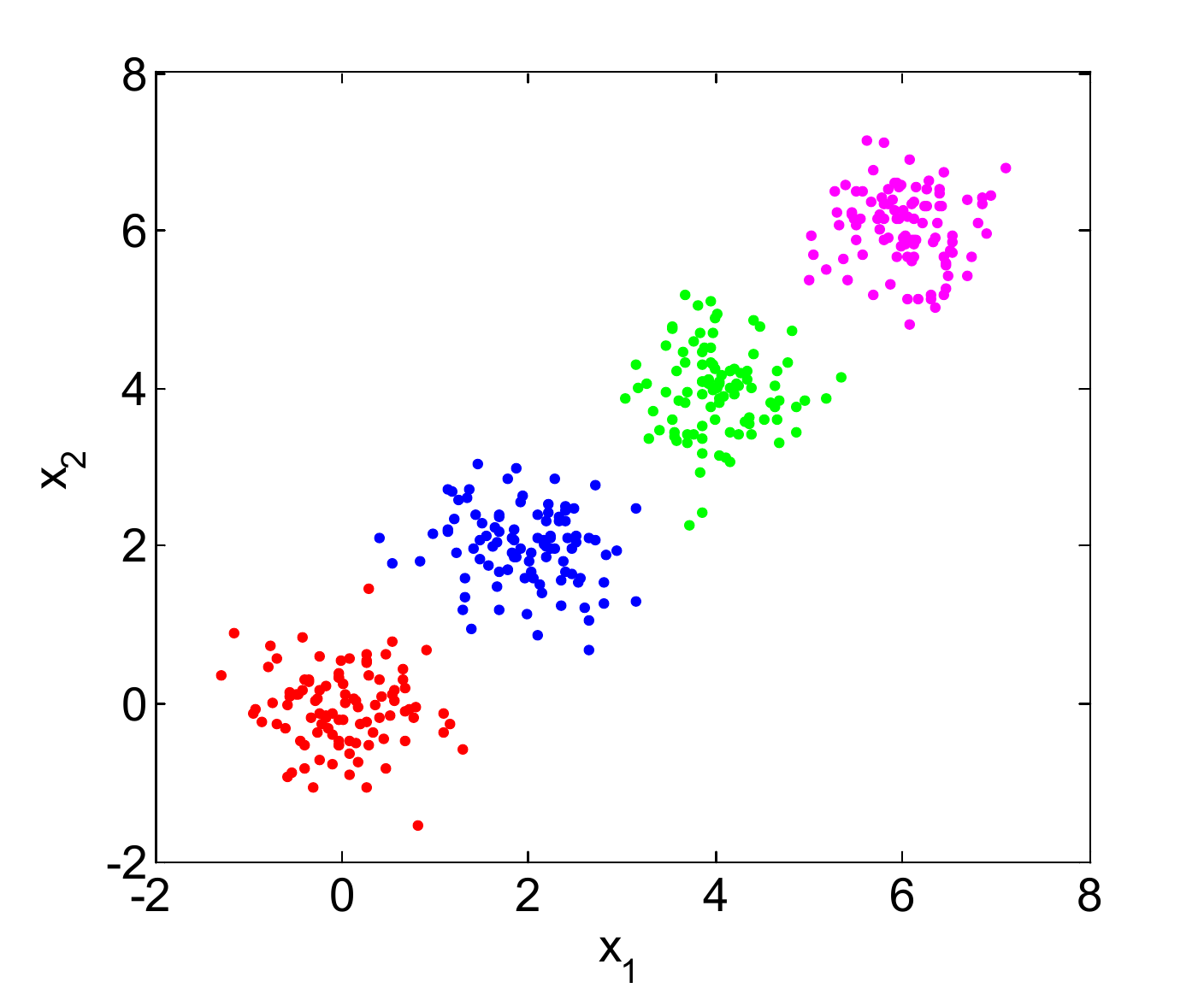} 
 \caption{Two-dimensional separable dataset with four classes. Each class contains 100 data. Data in four classes were drawn from normal distribution having variance 0.5 and mean 0, 2, 4 and 6, respectively. $x$-axis and $y$-axis are the first and the second feature of the data.}\label{fig:4lei_data}
\end{figure}

We adopt separability of dimensionality reduced data to measure the feature extraction performance of the SE and the S-SSE. The separability metric is the ratio of between-class distance and within-class distance, i.e.
\begin{equation*}
J=\frac{\mathrm{tr}(S_b)}{\mathrm{tr}(S_w)},
\end{equation*}
where $S_w=\sum_{i=1}^c \mathbf{P}_i\frac{1}{N_i}\sum_{j=1}^{N_i} (\mathbf{x}_j^{(i)}-\mathbf{s}_i)(\mathbf{x}_j^{(i)}-\mathbf{s}_i)^\top$ is the within-class dispersion matrix, $S_b=\sum_{i=1}^c \mathbf{P}_i (\mathbf{s}_i-\mathbf{s})(\mathbf{s}_i-\mathbf{s})^\top$ is the between-class dispersion matrix, $c$ is the number of classes, $\mathbf{P}_i$ is the priori probability of the $i$-th class, $N_i$ is the number of samples contained in the $i$-th class, $\mathbf{x}_j^{(i)}$ is the $j$-th sample in the $i$-th class, $\mathbf{s}_i$ is the mean of samples in the $i$-th class, $\mathbf{s}$ is the mean of all samples. The larger the $J$ is, the better the separability is. 

\begin{figure}[!h]
\centering
\subfigure[]{
    \label{fig:kefen_mean}
  \includegraphics[width=0.48\textwidth]{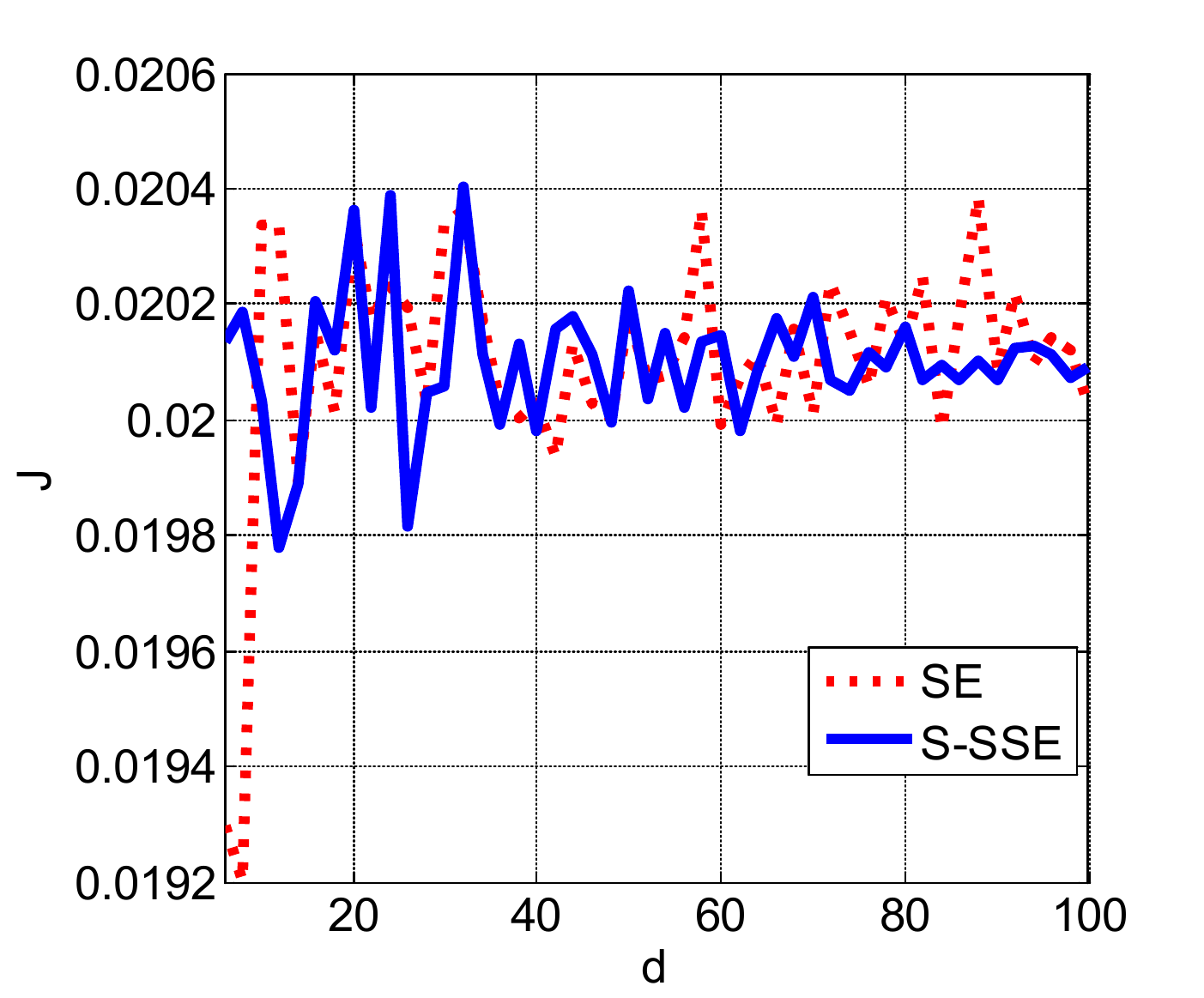}} 
  \subfigure[]{
    \label{fig:kefen_var}
  \includegraphics[width=0.48\textwidth]{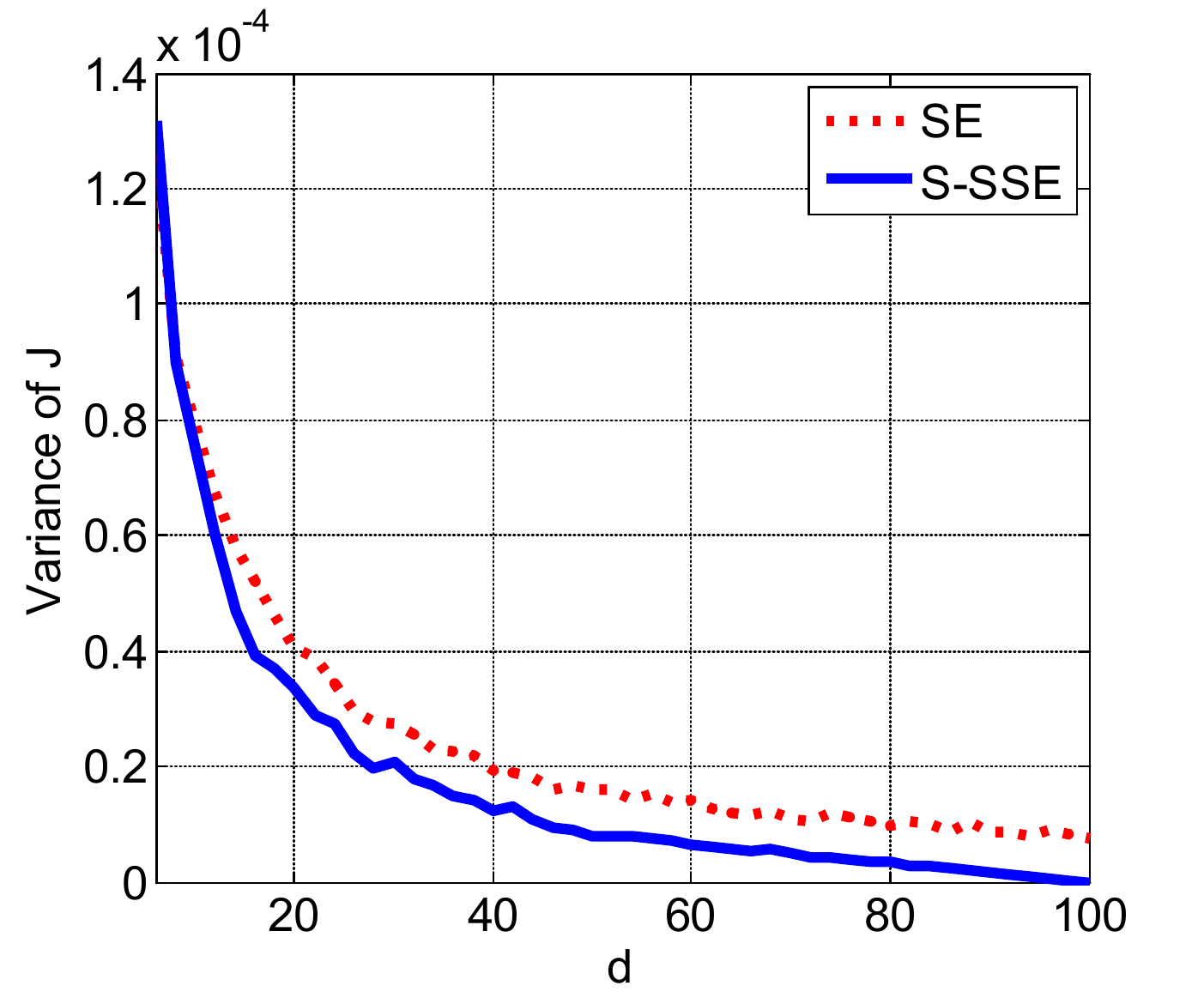}} 
  \caption{(a) Separability comparison. $y$-axis is the separability measurement $J$. (b) Variance of separability comparison. $y$-axis is the variance of $J$. Original dimension is 100. $x$-axis is the reduced dimension $d$.}
  \label{fig:separability}
\end{figure}

In order to obtain unbiased results, we ran programs 1000 times independently for each dimension $d$ and computed mean and variance of $J$. Fig. \ref{fig:separability} gives the experimental results. Fig. \ref{fig:kefen_mean} illustrates that values of $J$ are fluctuated around $0.0201$ for different $d$, yet the range of fluctuation is small, about $\pm 0.0003$, which illustrates that the separability of the data is still good after dimensionality reduction by using the SE and the S-SSE. We can also observe that values of $J$ at some $d$ are larger than that at $d=100$, which indicates that feature extraction may improve the separability of the data. With the increasing of $d$, the fluctuation of $J$ decreases, and more and more close to the value of $J$ at $d=100$, which indicates that the separability of dimensionality reduced data becomes stabler as $d$ increases. The fluctuation of $J$ for the S-SSE is smaller than that for the SE, which indicates that the separability of the data dimensionality reduced by using the S-SSE method is stabler than that by using the SE method. Fig. \ref{fig:kefen_var} shows that the variances of $J$ for the SE and the S-SSE both decrease as the dimension increases, which indicates that the larger the reduced dimension is, the stabler the data separability is. For all the $d$, the variances of $J$ for the S-SSE are all smaller than that for the SE, which indicates that the S-SSE is stabler than the SE. Overall, the S-SSE is able to maintain data separability as the SE, but the S-SSE is stabler than the SE, because the random matrix constructed by the S-SSE method is stabler.

\subsection{Euclidean distance preservation comparison}\label{sec:exp_preserve}
\subsubsection{The variation of relative error $\epsilon$ with $d$}
In order to compare the preservation of Euclidean distance for the SE and the S-SSE, we conducted experiments on data with 1000 dimensions to measure the variation of relative error $\epsilon=|\frac{\|xR\|_2}{\|x\|_2}-1|$ with reduced dimension $d$. Entries in the data were randomly chosen from $[0,1]$ or standard normal distribution with mean 0 and variance 1, because real-world datasets are usually normalized to these two distributions before training. The dimension was reduced from 1000 to $d$, where $d$ was set as $20$ to $200$ with interval $20$. For every $d$, experiments were performed 100 times independently and the mean of $\epsilon$ was calculated to obtain unbiased results. Fig. \ref{fig:distance} gives the experimental results. It can be shown from Fig. \ref{fig:distance} that $\epsilon$ decreases with the increasing of $d$. This is consistent with reality. Moreover, the relative error of the S-SSE is less than that of the SE in most cases. Therefore, the S-SSE can preserve the Euclidean distance better than the SE after dimensionality reduction.
\begin{figure}[!ht]
\centering
  \subfigure[]{
    \label{fig:norm} 
    \includegraphics[width=0.48\textwidth]{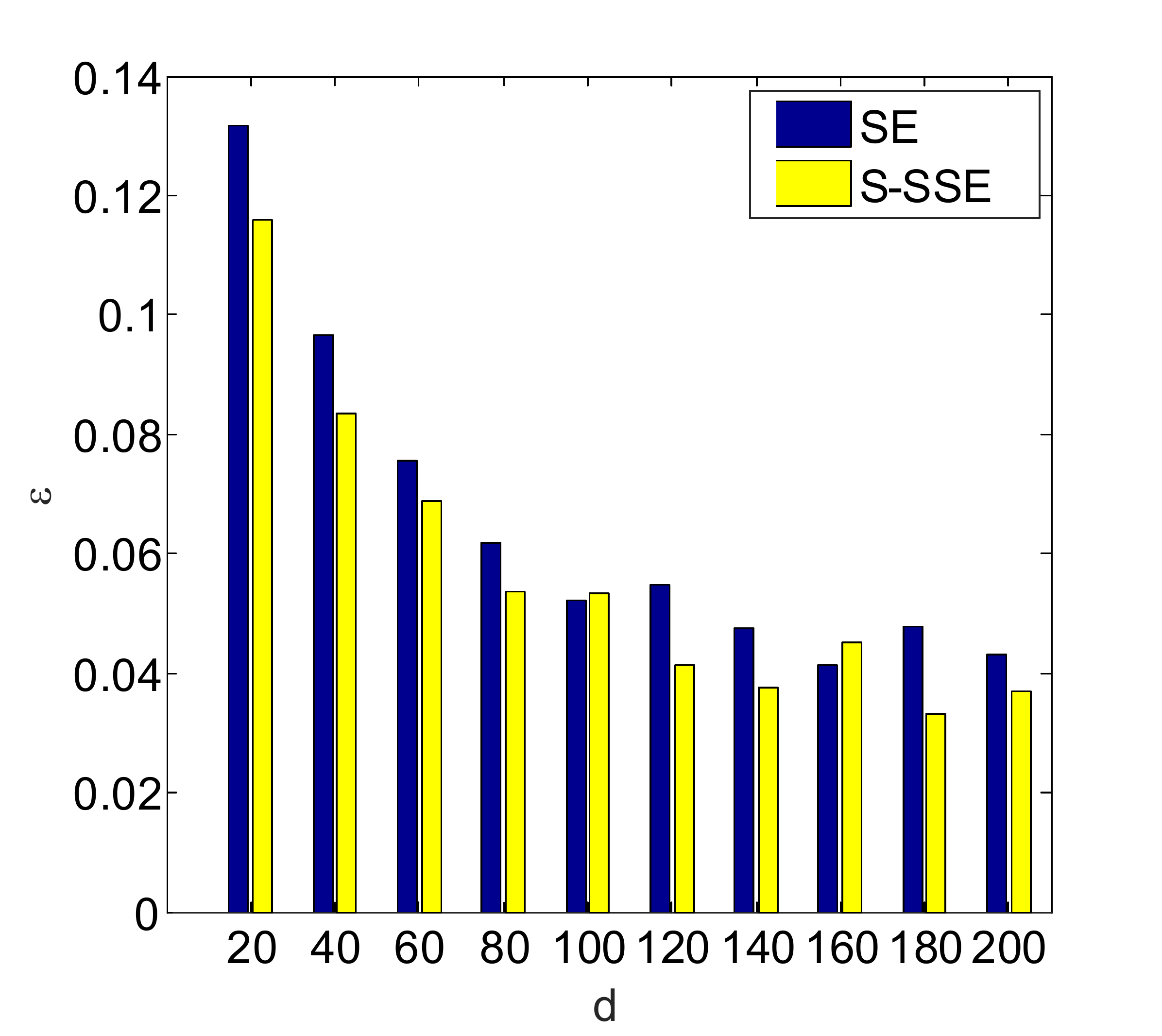}} 
  \subfigure[]{
    \label{fig:0-1} 
    \includegraphics[width=0.48\textwidth]{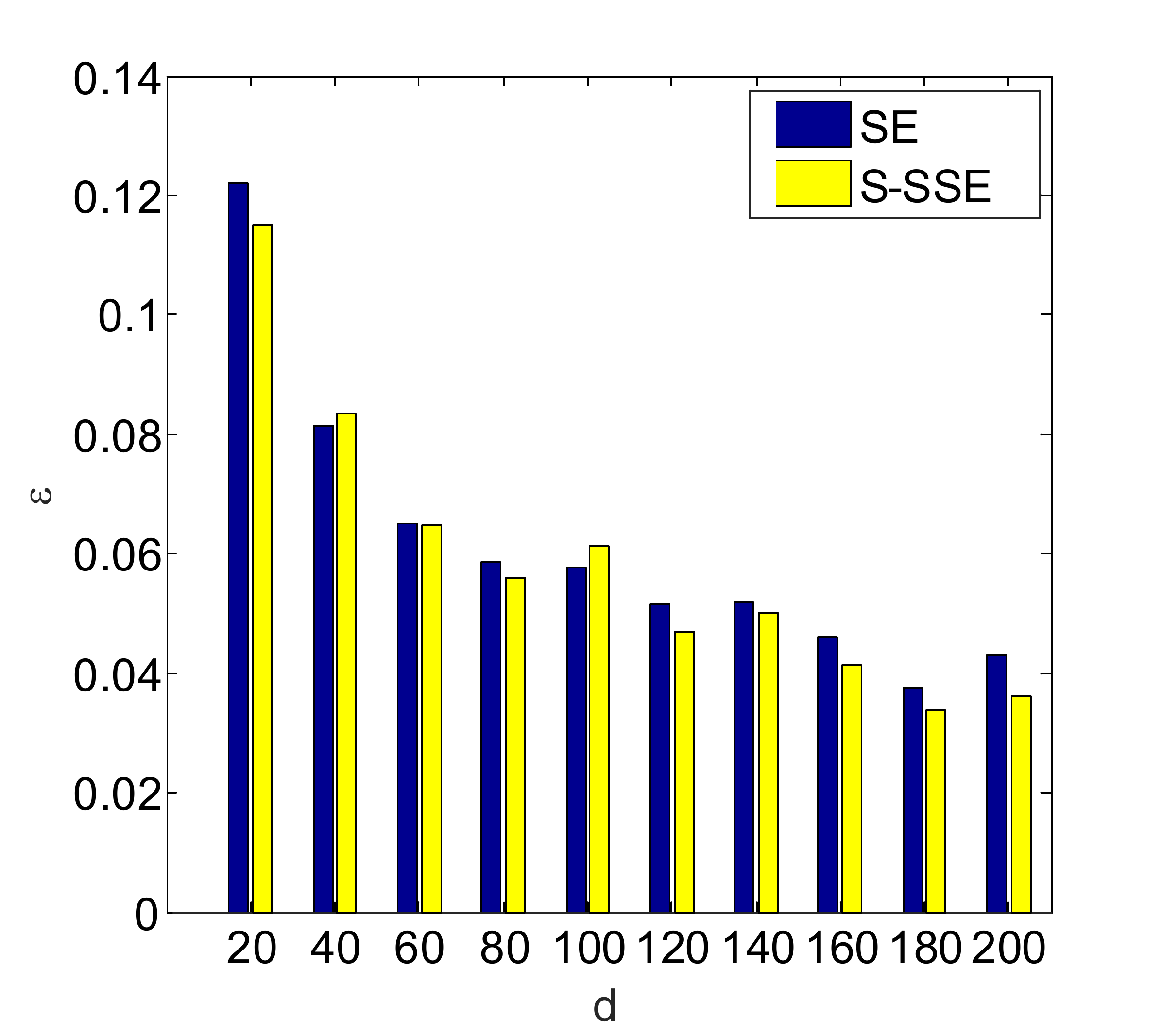}}
  \caption{Compare the variety of relative error $\epsilon$ with $d$. $\epsilon=|\frac{\|R \mathbf x\|_2}{\|\mathbf x\|_2}-1|$. Original data contain 1000 features with each feature randomly generated from $[0,1]$ and standard normal distribution with mean 0 and variance 1 for (a) and (b), respectively. $d$ is the reduced dimension, and $\epsilon$ is the mean of relative errors of 100 trials. }
 \label{fig:distance}
 \end{figure}
 \subsubsection{The variation of distance preservation probability $p$ with $d$} \label{sec:1-delta_d}
 In order to verify the conclusion of Theorem \ref{thm:preserve}, and further compare the preservation of Euclidean distance after dimensionality reduction by the SE and the S-SSE, experiments were conducted on one synthetic dataset and two benchmark datasets. We calculate frequency of $\|R\mathbf x\|_2$falling within the interval $[(1-\epsilon)\|\mathbf x\|_2, (1+\epsilon)\|\mathbf x\|_2]$. Experiments were run 10,000 times independently and computed the mean of the frequencies as the distance preservation probability. For convenience, we denote this probability value as $p$, i.e. $p:=\mathbb{P}((1-\epsilon)\|\mathbf{x}\|_2\leq \|R\mathbf{x}\|_2\leq(1+\epsilon)\|\mathbf{x}\|_2)$, which is related to $\epsilon$ and $R$. If $\epsilon$ is fixed at a constant, then the larger $p$ is, the better the Euclidean distance preservation of $R$ is.

The synthetic dataset contains 1000 samples with dimension 200, which were uniformly and randomly generated from interval $[0,1]$. The benchmark datasets are DNA and MADELON, whose information is listed in Table \ref{tab:data}. To measure the variation of distance preservation probability $p$ with $d$, $\epsilon$ was fixed at $\epsilon=0.1\in (0,0.5)$, and $d$ was set as $20$ to $200$ with interval $20$. Fig. \ref{fig:d-delta} gives the experimental results. Fig. \ref{fig:d-delta} illustrates that as $d$ increases, $p$ also increases gradually approaching to 1, which indicates that the distance preservation probability increases with the increasing of reduced dimension. With regard to the same $d$, the value of $p$ for the S-SSE is larger than that for the SE, which indicates that the S-SSE method can better preserve Euclidean distance approximation.
\begin{figure}[!h]
\centering
  \subfigure[Synthetic Dataset]{
    \label{fig:d_1-delta} 
    \includegraphics[width=0.31\textwidth]{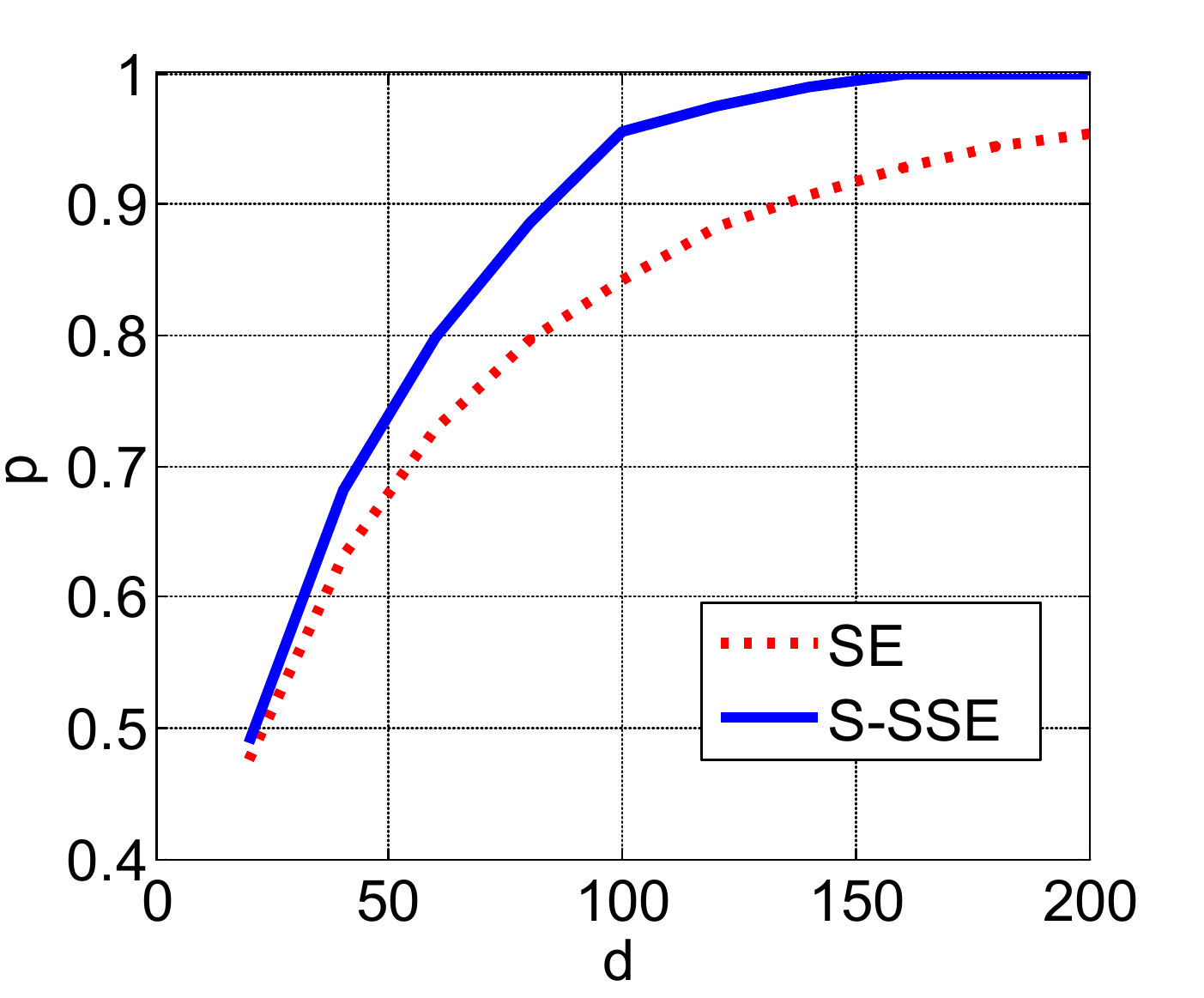}} 
  \subfigure[DNA]{
    \label{fig:dna_d_1-delta} 
    \includegraphics[width=0.31\textwidth]{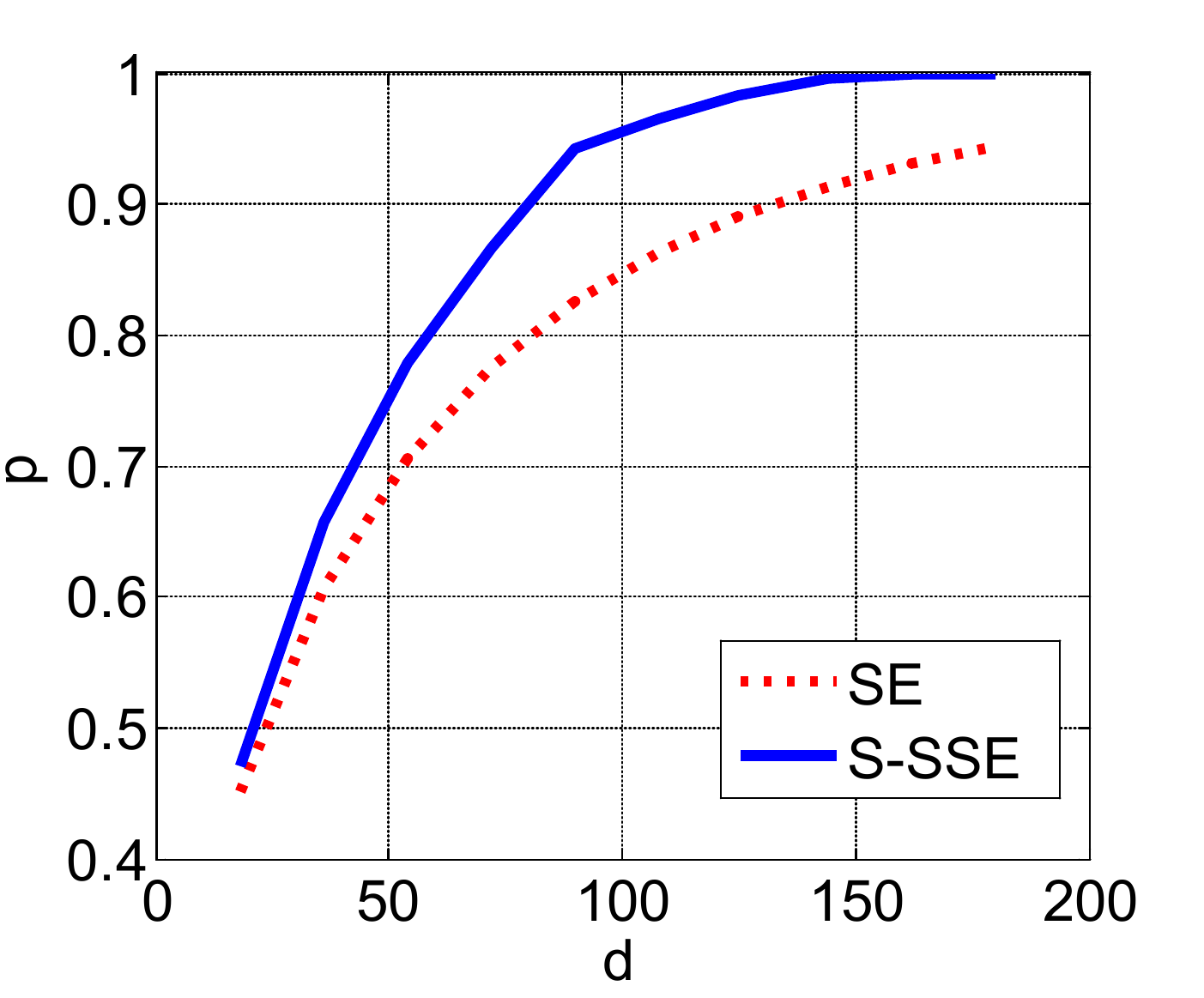}}
  \subfigure[MADELON]{
    \label{fig:MADELON} 
    \includegraphics[width=0.31\textwidth]{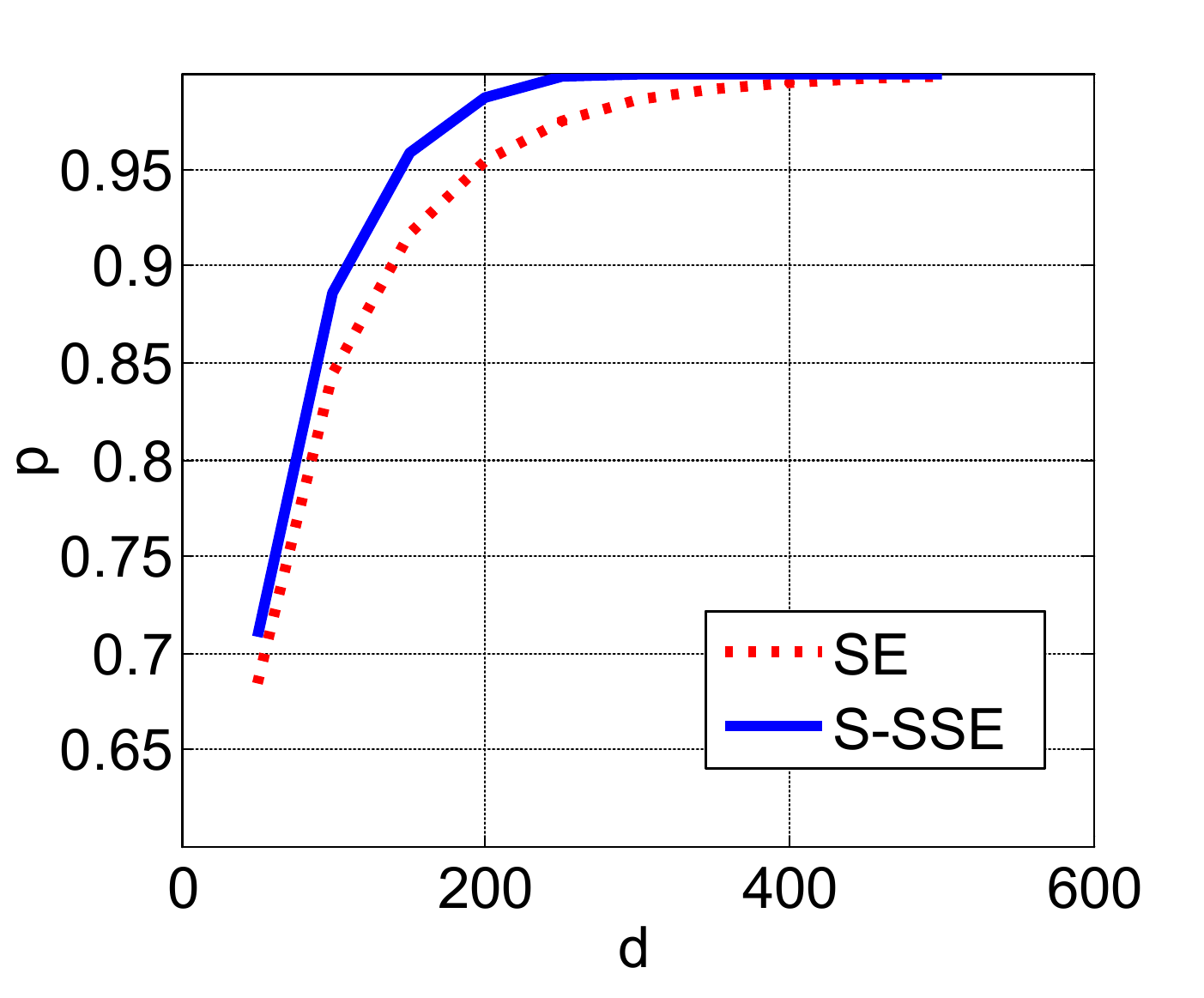}}
  \caption{The variation of $p$ with $d$ on the dimensionality reduced data generated by the SE and the S-SSE. $p$ is the distance preservation probability. Set $\epsilon=0.1$. (a), (b) and (c) are the results of synthetic dataset, DNA and MADELON datasets, respectively.}
 \label{fig:d-delta}
 \end{figure}

 \subsubsection{The variation of distance preservation probability $p$ with $\epsilon$}
 To measure the relationship between distance preservation probability $p$ and relative error $\epsilon$, we fixed $d$ at $80$, $80$ and $100$ for synthetic dataset (the generation method is the same as that in subsection \ref{sec:1-delta_d}), DNA and MADELON, respectively. $\epsilon$ was set as $0.05$ to $0.5$ with interval $0.05$. The experiments were performed 10000 times independently and computed the mean of $p$ as the final results. Fig. \ref{fig:emsrong-delta} gives the experimental results. Fig. \ref{fig:emsrong-delta} shows that the values of $p$ gradually increase to 1 as $\epsilon$ increases, which indicates that with the enlarging of interval $[(1-\epsilon)\|\mathbf x\|_2, (1+\epsilon)\|\mathbf x\|_2]$, $p$ also increases, which is consistent with reality. The values of $1-p$ for the S-SSE are all smaller than 0.5, which indicates that the condition in Theorem \ref{thm:preserve} is reasonable. In addition, given the value of $\epsilon$, $p$ of the S-SSE method is larger than that of the SE method, which indicates that the probability of $\|R\mathbf x\|_2$ falling within the interval $[(1-\epsilon)\|\mathbf x\|_2, (1+\epsilon)\|\mathbf x\|_2]$ after dimension reduction by the S-SSE method is larger than that by the SE method, in other words, S-SSE method can better preserve Euclidean distance approximation.
  \begin{figure}[!h]
\centering
  \subfigure[Synthetic Dataset]{
    \label{fig:emsrong_1-delta} 
    \includegraphics[width=0.31\textwidth]{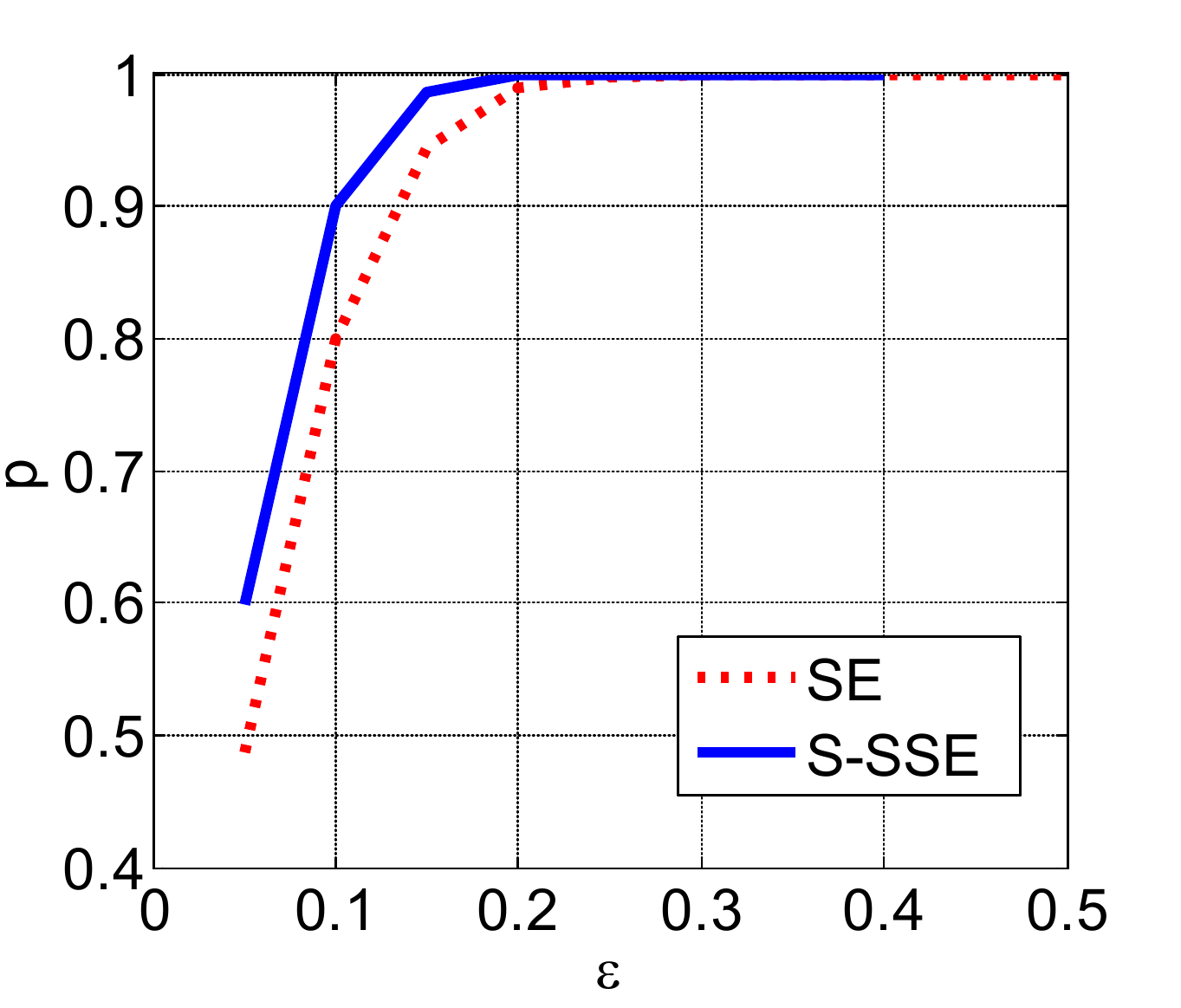}} 
  \subfigure[DNA]{
    \label{fig:dna_emsrong_1-delta} 
    \includegraphics[width=0.31\textwidth]{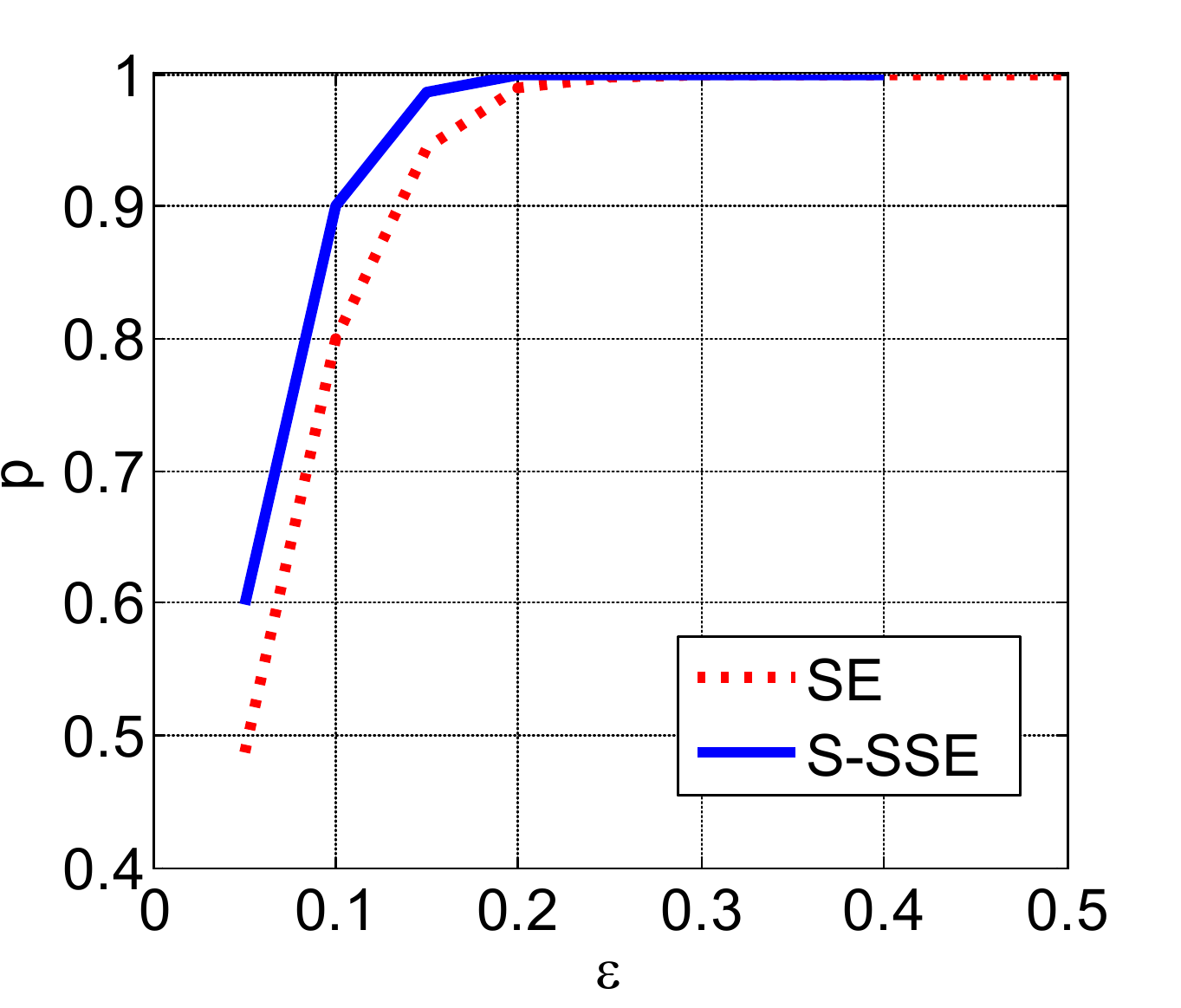}}
  \subfigure[MADELON]{
    \label{fig:madelon_emsrong_1-delta} 
    \includegraphics[width=0.31\textwidth]{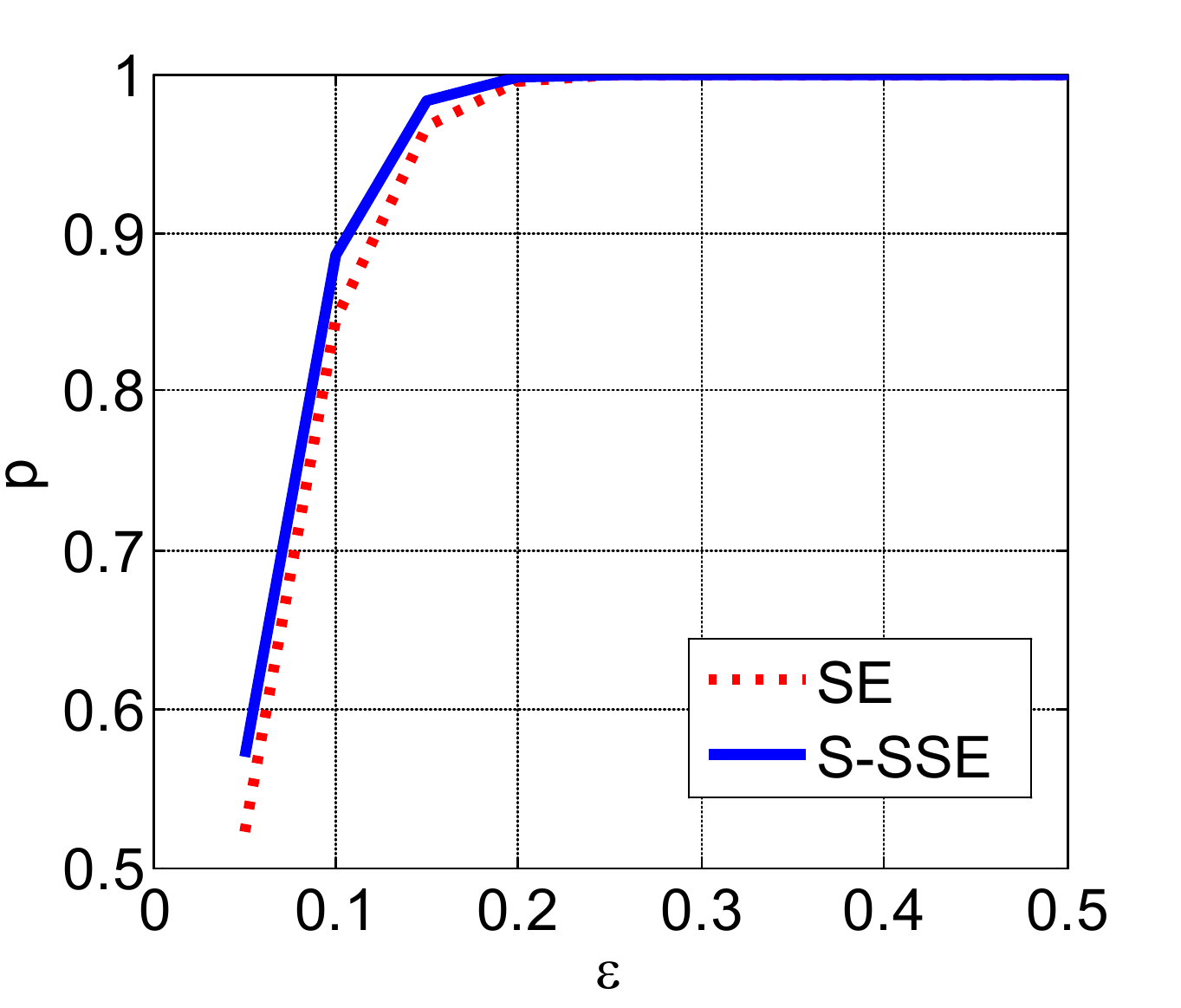}}
  \caption{The variation of distance preservation probability $p$ with $\epsilon$ on the dimensionality reduced data generated by the SE and the S-SSE. (a), (b) and (c) are the results of synthetic dataset, DNA and MADELON with $d=80, 80$ and $100$, respectively.}
 \label{fig:emsrong-delta}
 \end{figure}
 \subsection{$k$-means clustering experiments}
  Our S-SSE approach can be applied in various Euclidean distance based machine learning algorithms. In these algorithms, $k$-means clustering is one of the most widely used methods, but it is inefficient on dealing with high dimensional datasets. In order to evaluate the performance of the proposed feature extraction method applied in machine learning, this subsection uses the dimensionality reduced data onto the $k$-means clustering and compares the S-SSE against a few other prominent dimensionality reduction methods. For SPCA, we set the number of non-zero entries in each column of principal component directions matrix is 1 to compare the efficient of the SPCA, the SE and our S-SSE. The maximum number of iterations in SPCA is set 3000.  We can not get the results of SPCA within three days on GISETTE and SECTOR datasets. Thus, these results are not reported. The datasets can be downloaded from the LIBSVM website \footnote{\url{https://www.csie.ntu.edu.tw/~cjlin/libsvmtools/datasets/}}. Table \ref{tab:data} lists the information of the datasets, including the number of samples, features and classes.
 \begin{table}[!h]
  \centering
  \caption{Information of datasets used in the experiments}
  \begin{small}
  \begin{tabular}{lccc}
  \toprule
  Datasets&\#INSTANCE&\#FEATURES&\#CLASSES\\
  \midrule
  DNA&3186&180&3\\
  USPS&9298&256&10\\
  MADELON&2000&500&2\\
  MNIST&60000&780&10\\
  GISETTE&7000&5000&2\\
  SECTOR&9619&55197&105\\
  \bottomrule
\end{tabular}
\end{small}
\label{tab:data}
\end{table}

In order to compare the effect of feature extraction algorithms the SPCA, the DE, the SE, the SPCA and the S-SSE, we ran standard $k$-means clustering algorithm after dimensionality reduction. We also compare all these algorithms against the standard $k$-means clustering algorithm on the full dimensional datasets. In experiments, Cai's Litekmeans package \footnote{\url{http://www.zjucadcg.cn/dengcai/Data/Clustering.html}} performs very well, hence we employed Cai's package in our experiments. The results in the figures are the mean of ten runs for each dataset. In each run, $k$-means clustering repeats twenty times, each with a new set of initial centroids, and returns the best one as the clustering output, i.e. in MATLAB, we ran the following command:  \emph{litekmeans(X, $k$, 'Replicates', 20)}.
\subsubsection{Evaluation methodology}
To measure the quality of all the methods, we reported the clustering accuracy \cite{fahad2014survey}, e.g. $accuracy=0.9$ implies that $90\%$ of the points are assigned the ``correct cluster".
We also reported the running time (in seconds) of constructing the matrix $R$ and computing the multiplication $RX^\top$ for all the compared algorithms. All the reported results correspond to the average values of 10 independent runs.

\subsubsection{Results}

Experimental results are shown in Figs. \ref{fig:acc} - \ref{fig:time_XR}. $x$-axis is compression factor, i.e. the ratio of the number of features after reduction and the number of original features, for instance, $compression~factor=0.3$ indicates that we extract $30\%$ of original features. For SECTOR, the maximum compression factor is set as $0.4$ because its dimension is so extremely high that training it consumes excessive memory.
  \begin{figure}[!h]
\centering
 \subfigure[DNA]{
    \label{fig:dna_acc} 
    \includegraphics[width=0.31\textwidth]{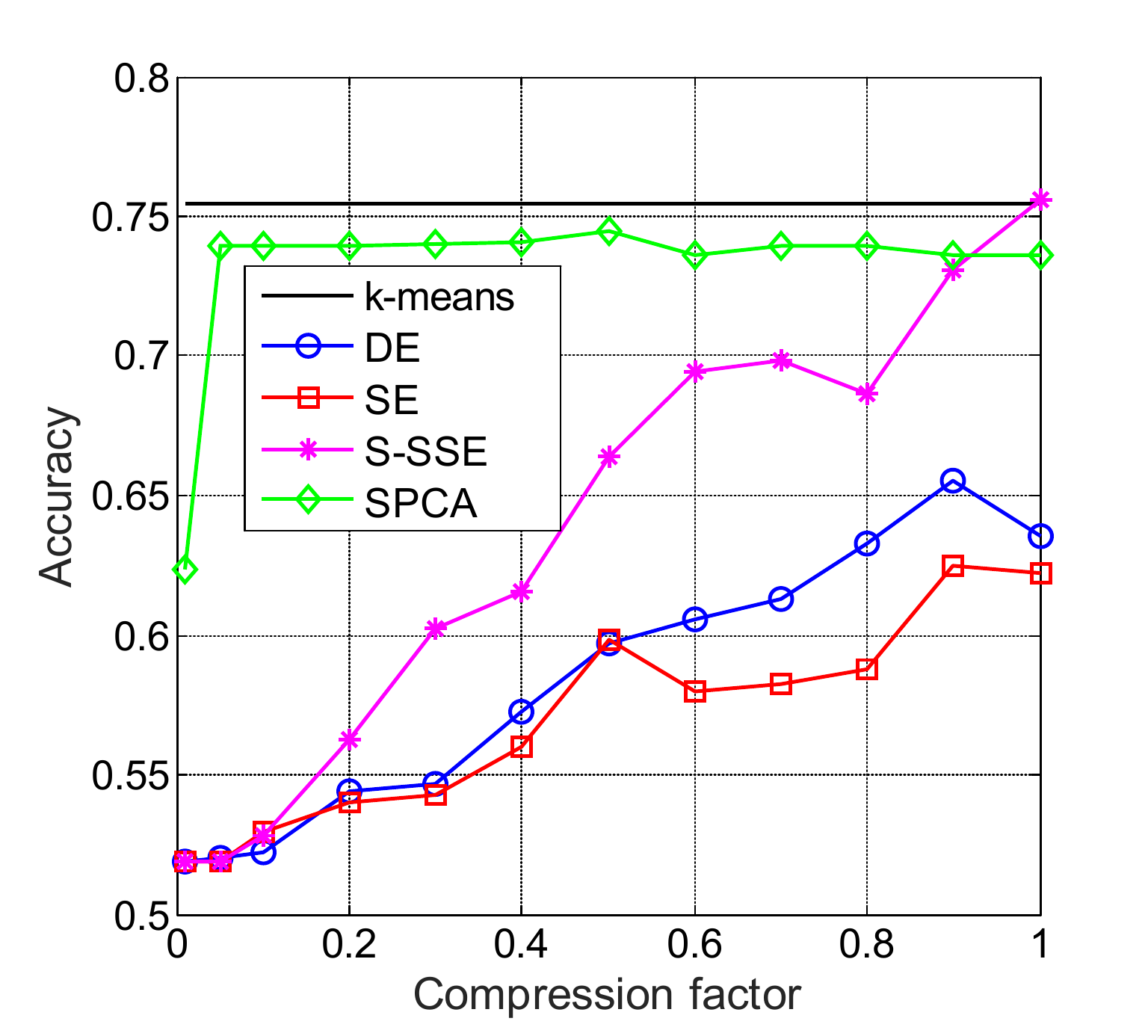}} 
 \subfigure[USPS]{
    \label{fig:usps_acc} 
    \includegraphics[width=0.32\textwidth]{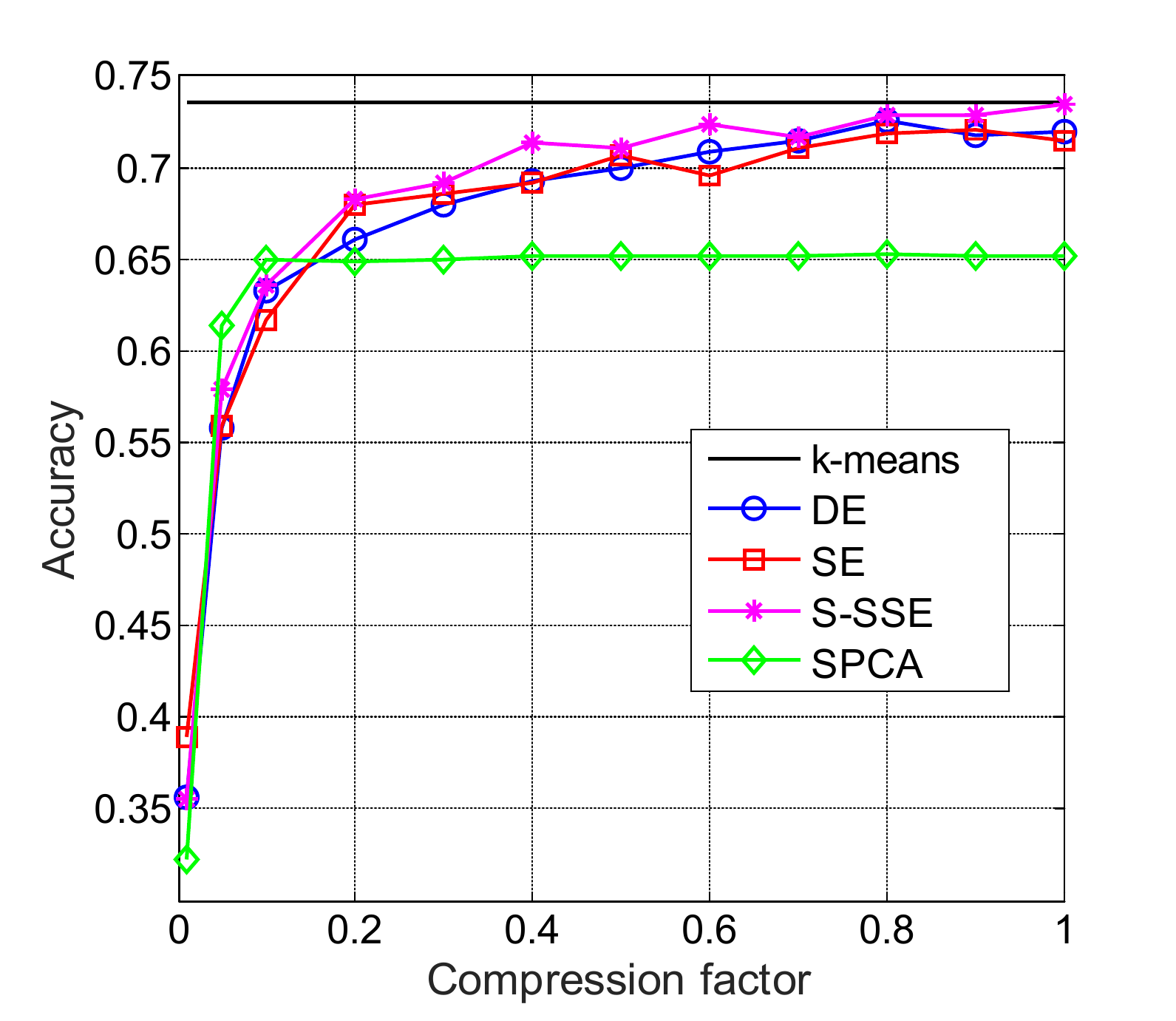}}
 \subfigure[MADELON]{
    \label{fig:mushrooms_acc} 
    \includegraphics[width=0.31\textwidth]{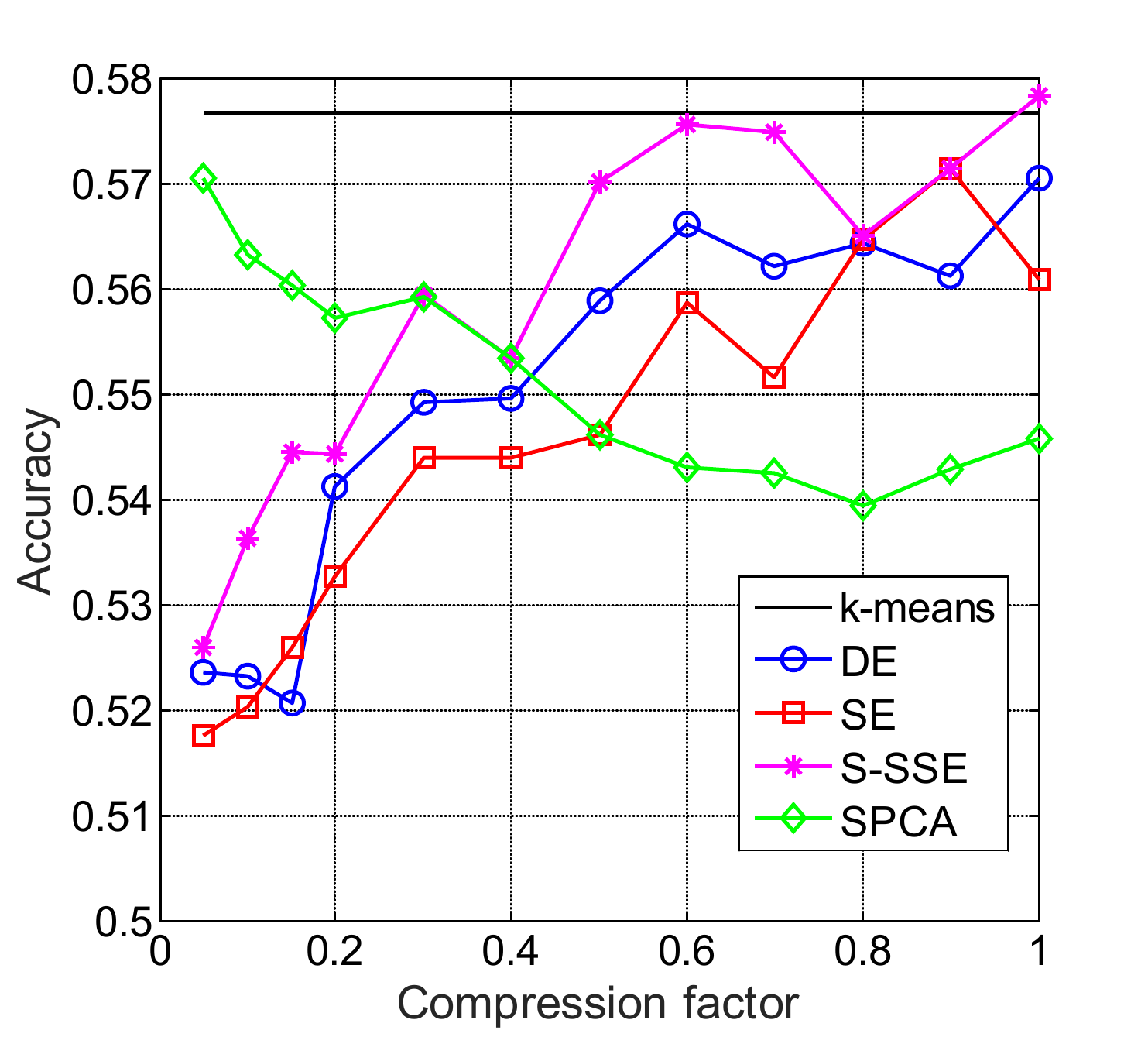}}
  \vfill
 \subfigure[MNIST]{
    \label{fig:mnist_acc} 
    \includegraphics[width=0.31\textwidth]{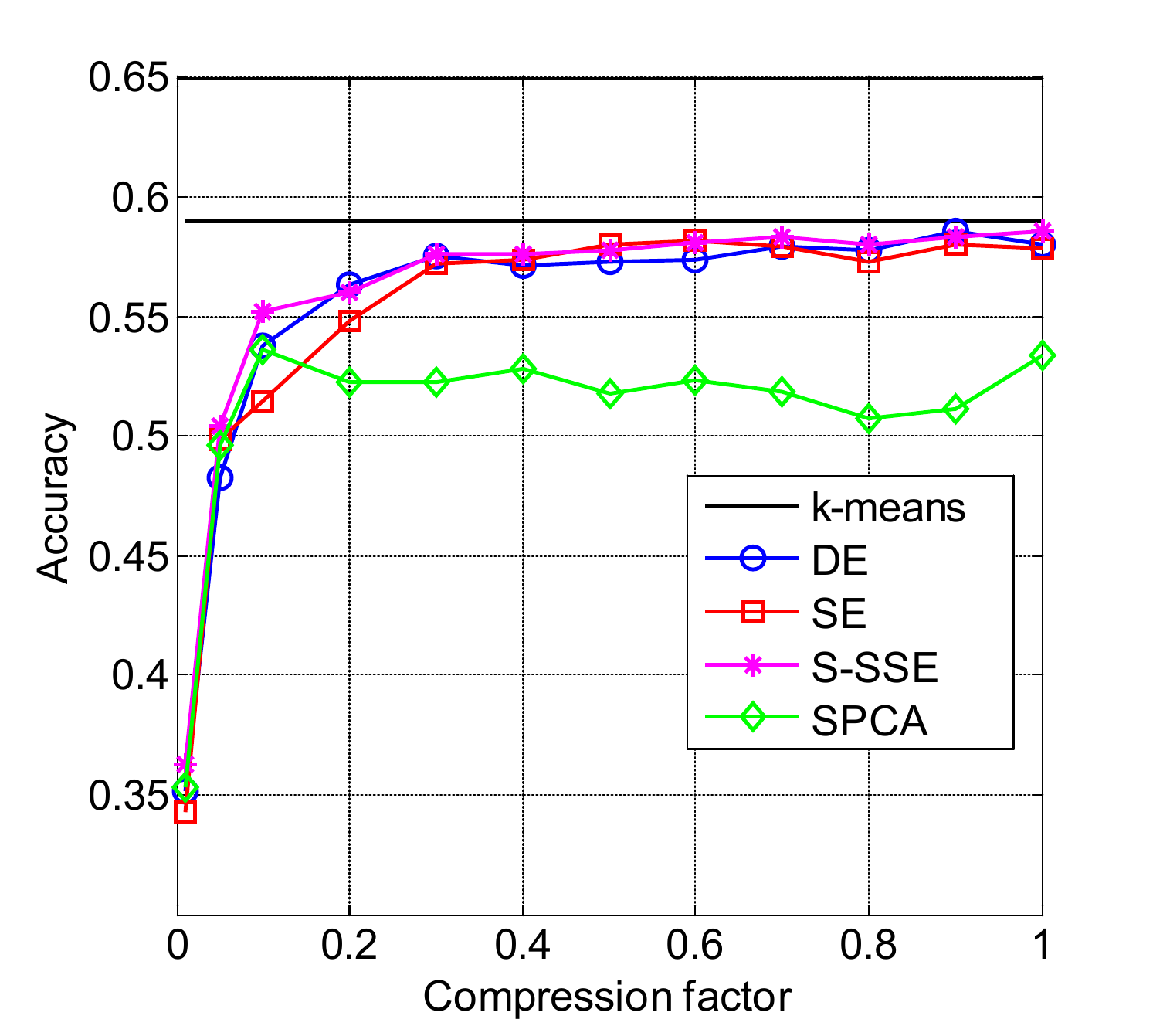}}
 \subfigure[GISETTE]{
    \label{fig:gisette_acc} 
    \includegraphics[width=0.31\textwidth]{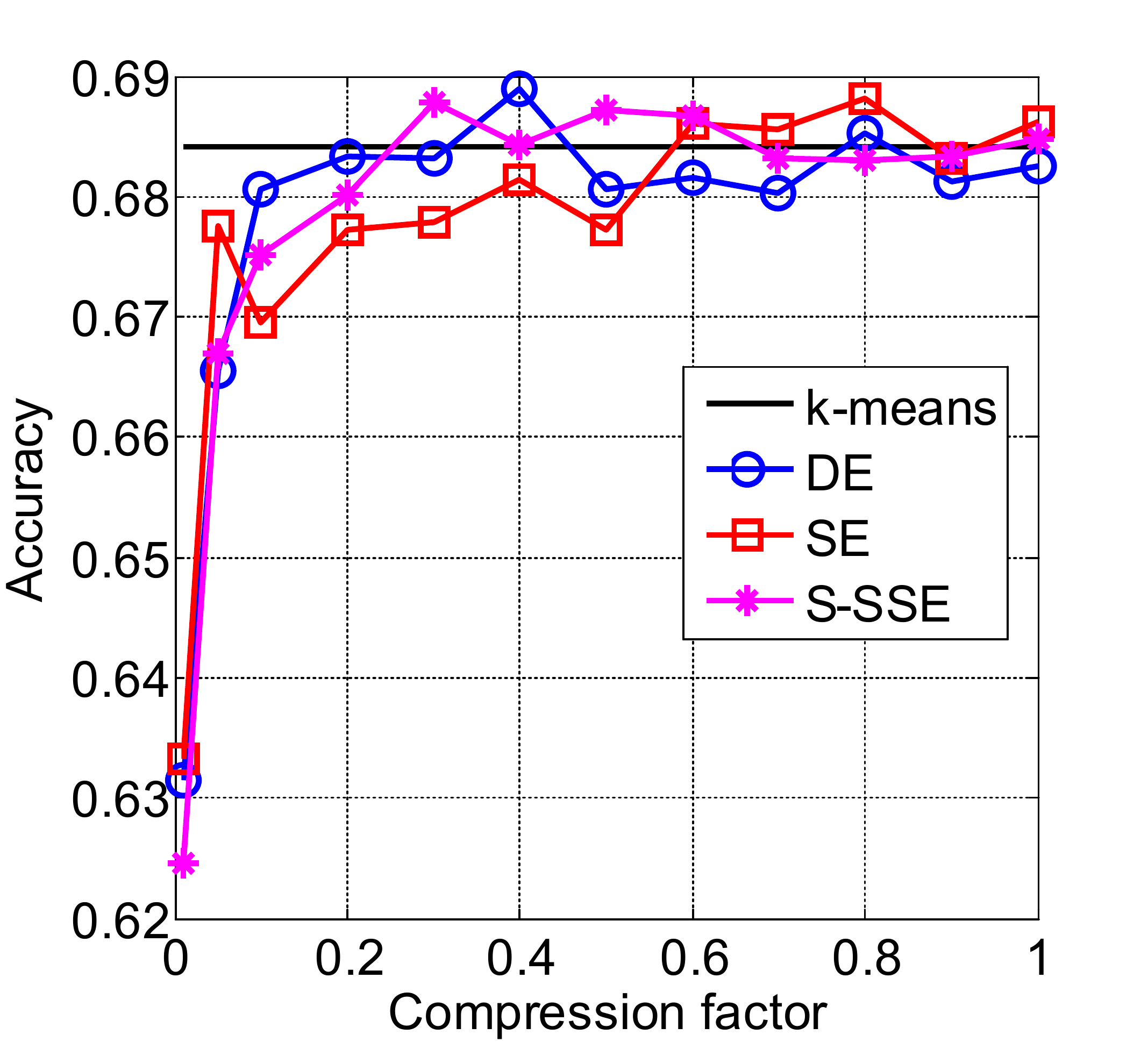}}
 \subfigure[SECTOR]{
    \label{fig:sector_acc} 
    \includegraphics[width=0.32\textwidth]{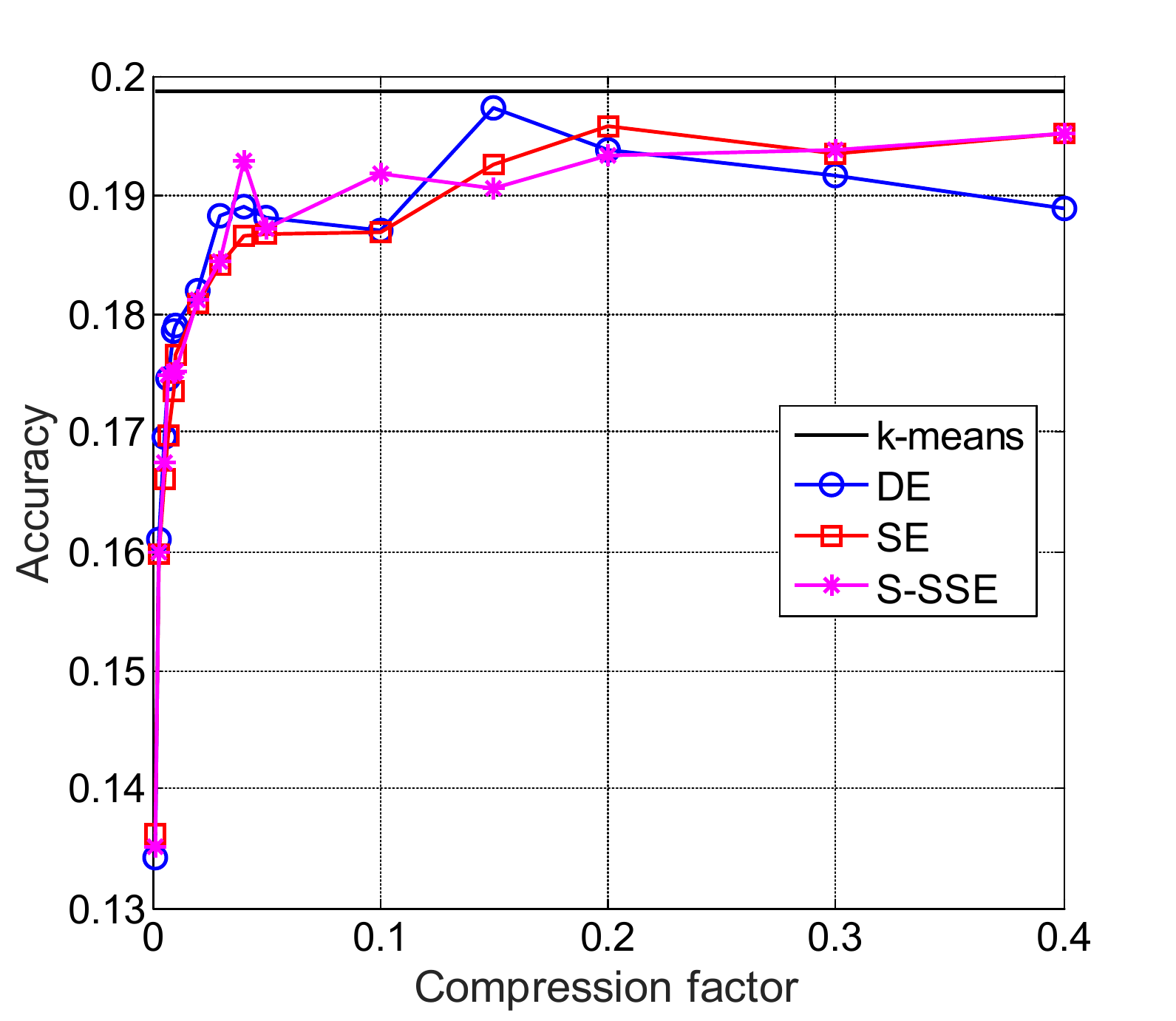}}
 \caption{Clustering accuracy for various dimensionality reduction methods on six real-world datasets. $x$-axis is the compression factor, i.e. the ratio of the number of features after reduction and the number of original features.}
 \label{fig:acc}
 \end{figure}

\begin{figure}[!h]
\centering
\subfigure[DNA]
    {\label{fig:dna_R} 
    \includegraphics[width=0.31\textwidth]{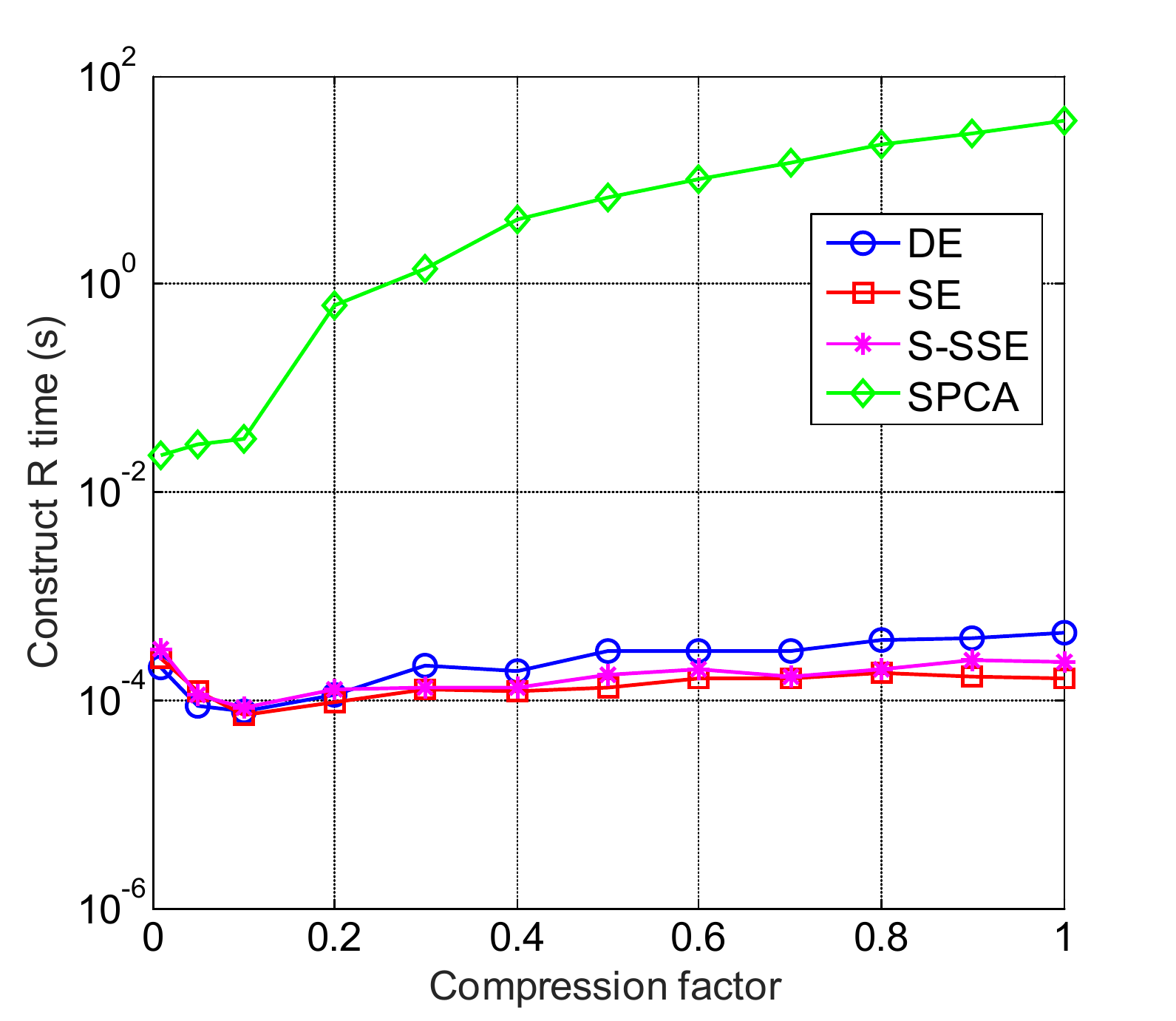}} 
\subfigure[USPS]{
    \label{fig:usps_R} 
    \includegraphics[width=0.31\textwidth]{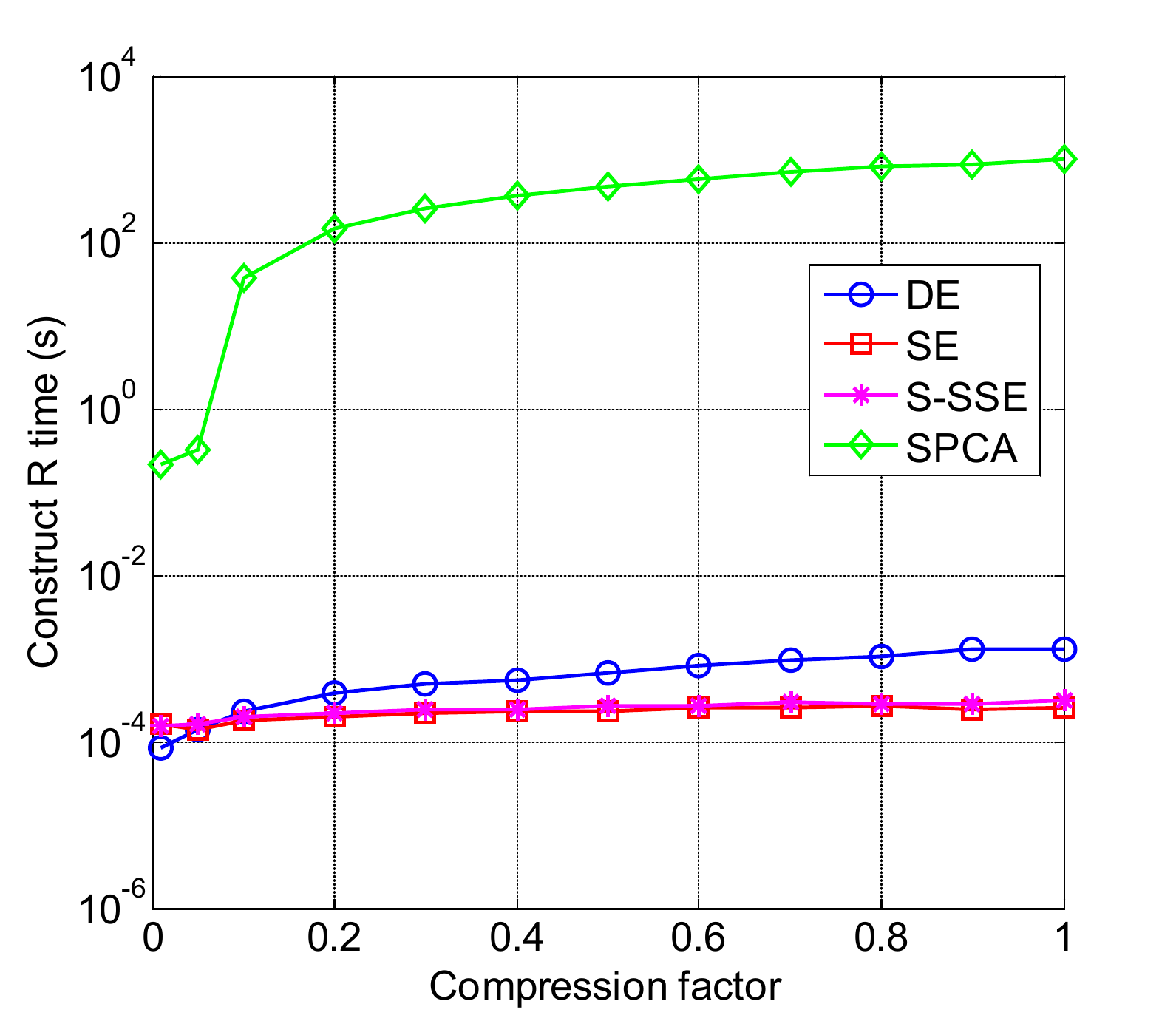}} 
\subfigure[MADELON]
    {\label{fig:madelon_R} 
    \includegraphics[width=0.31\textwidth]{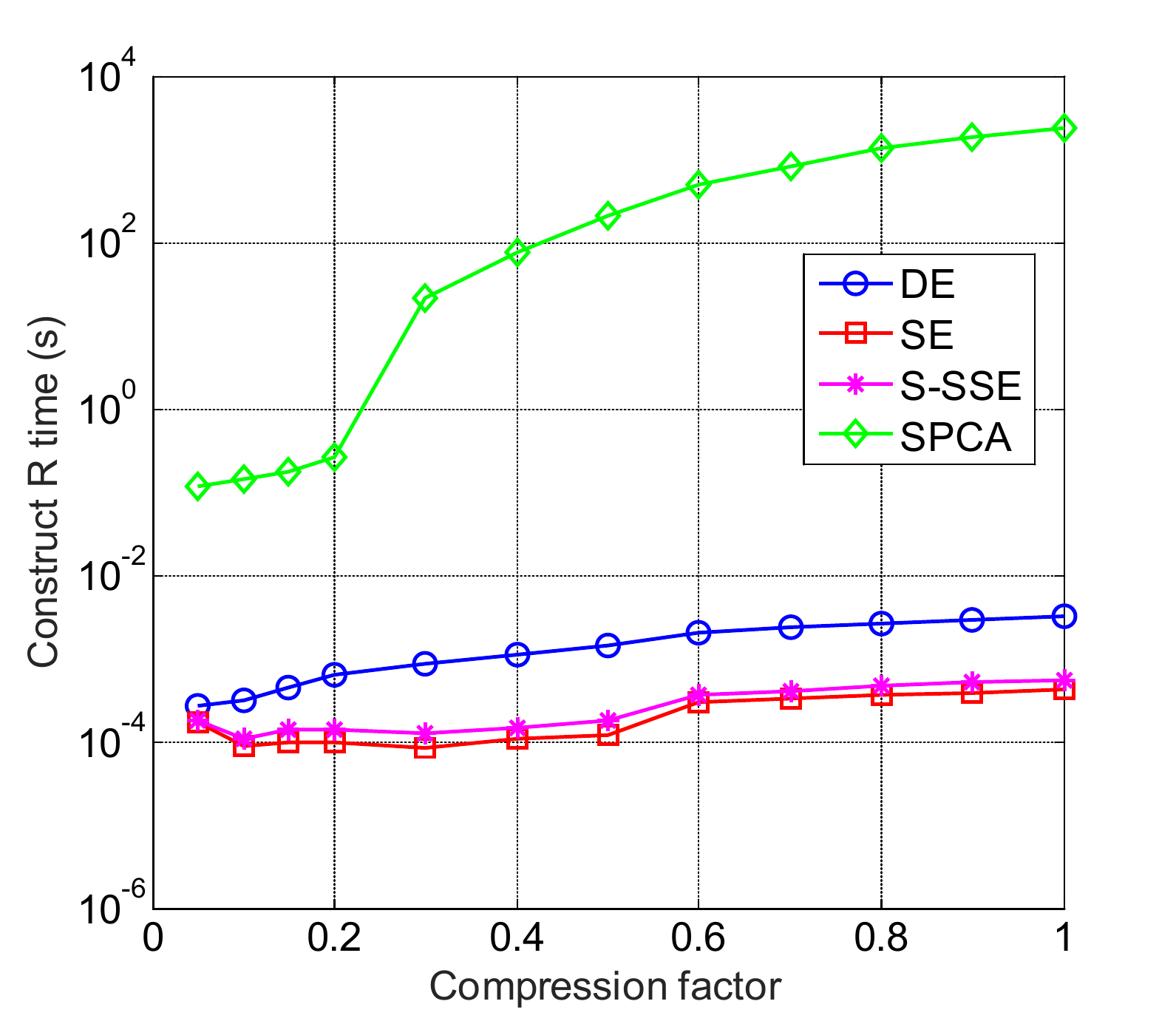}} 
 \vfill
 \subfigure[MNIST]{
    \label{fig:mnist_R} 
    \includegraphics[width=0.31\textwidth]{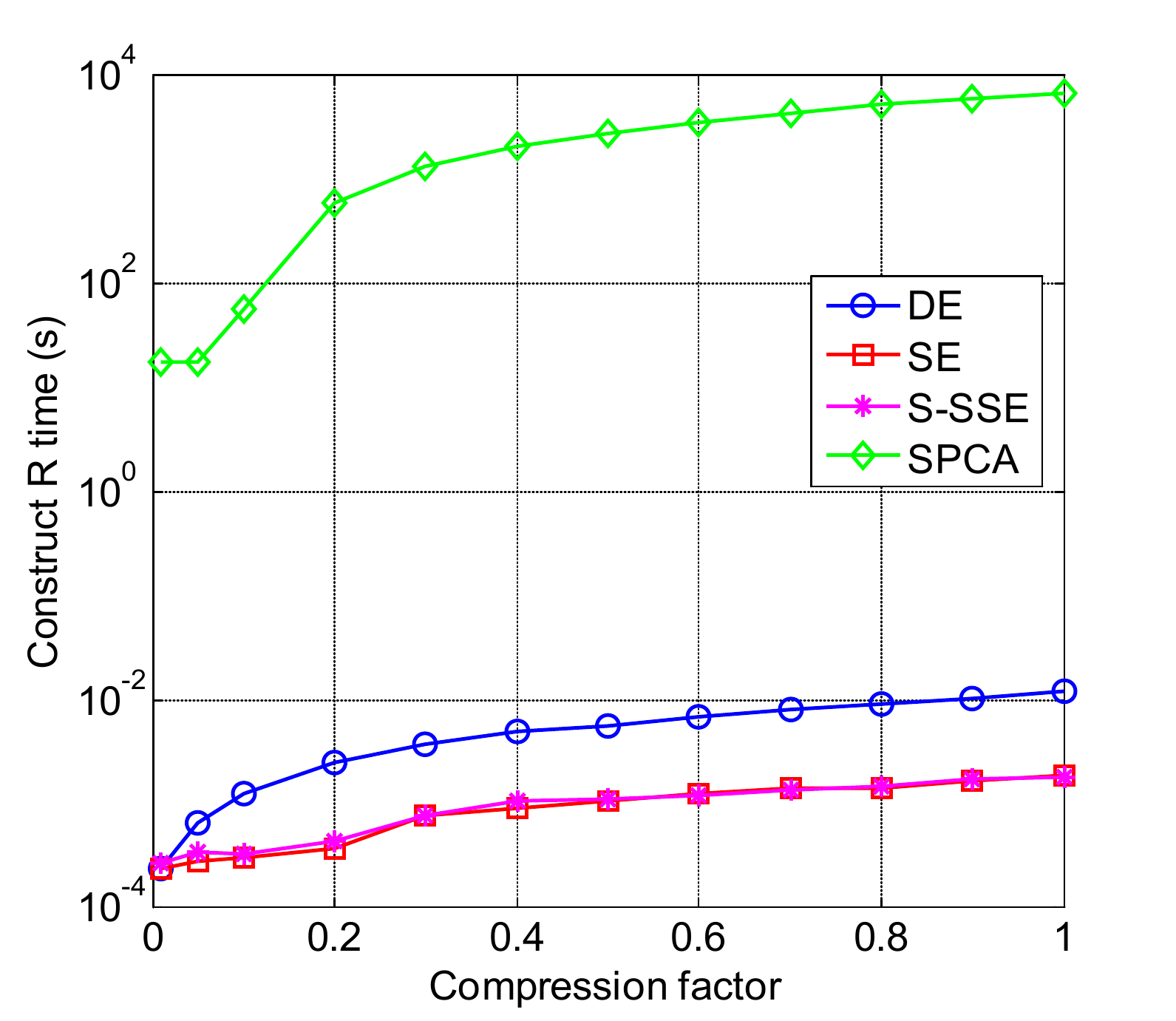}} 
 \subfigure[GISETTE]{
    \label{fig:gisette_R} 
    \includegraphics[width=0.31\textwidth]{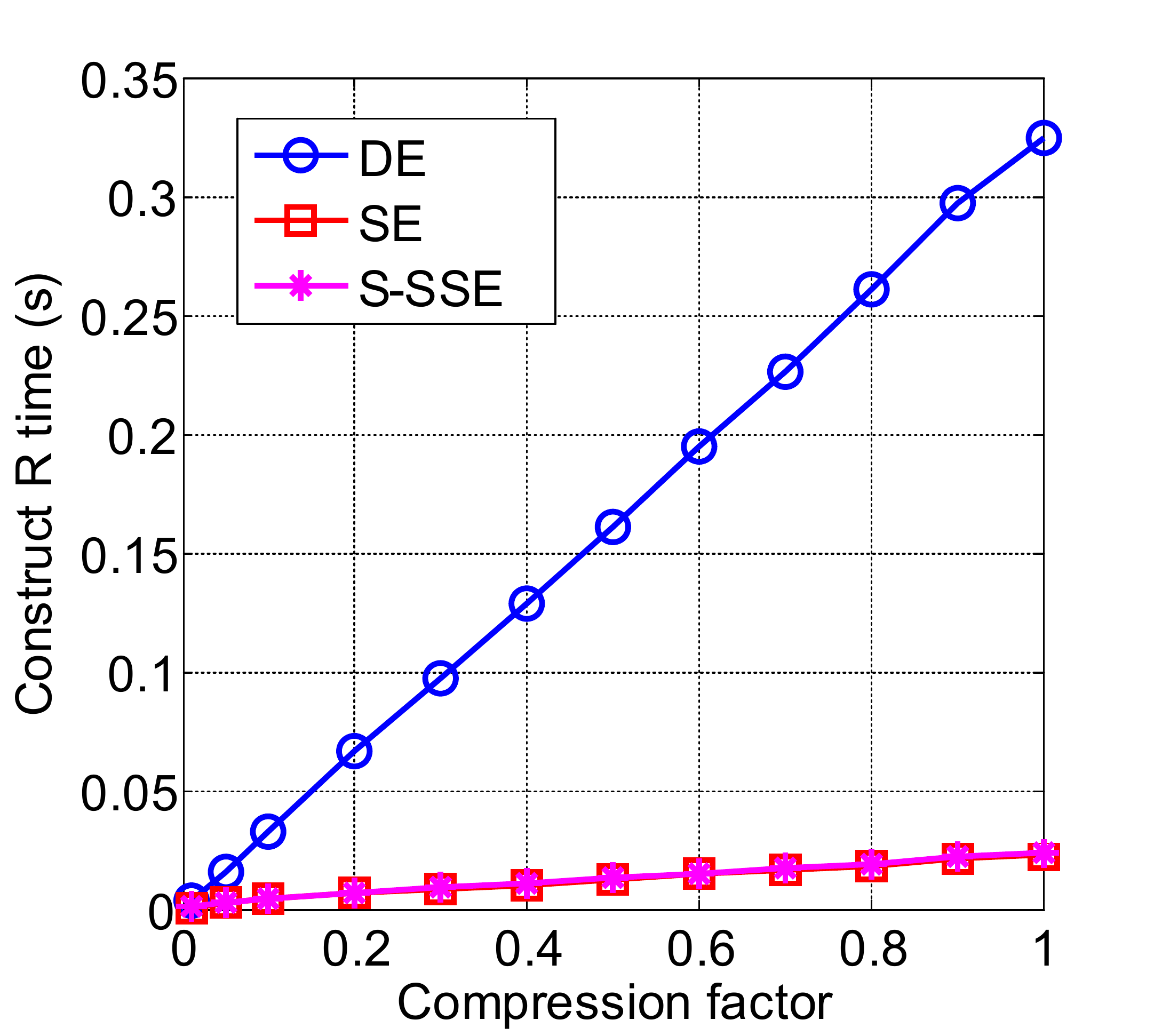}} 
 \subfigure[SECTOR]{
    \label{fig:sector_R} 
    \includegraphics[width=0.31\textwidth]{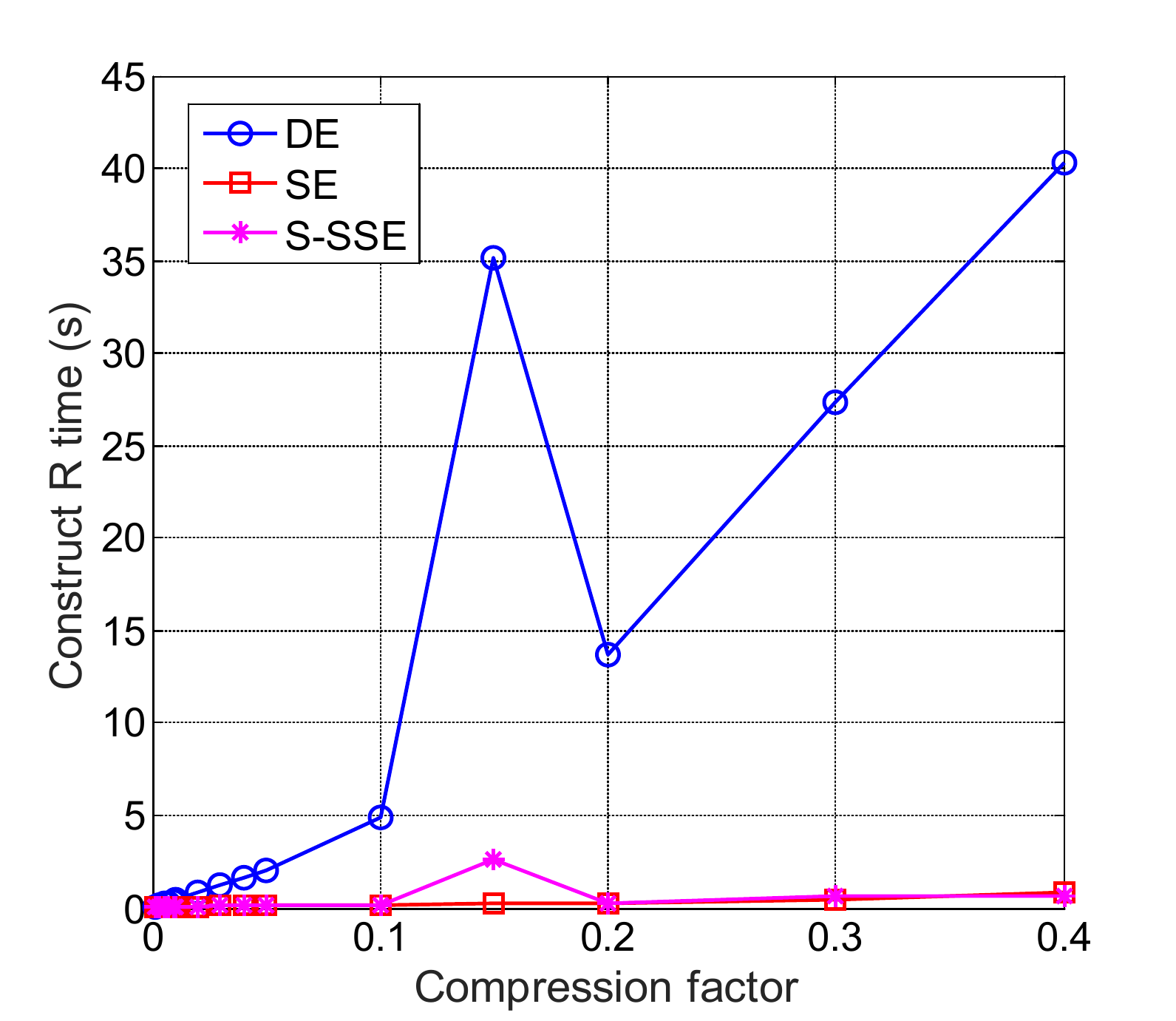}} 
 \caption{The time of constructing matrix $R$ for various methods.}
 \label{fig:time_R}
 \end{figure}

 \begin{figure}[!h]
\centering
 \subfigure[DNA]
    {\label{fig:dna_XR} 
    \includegraphics[width=0.31\textwidth]{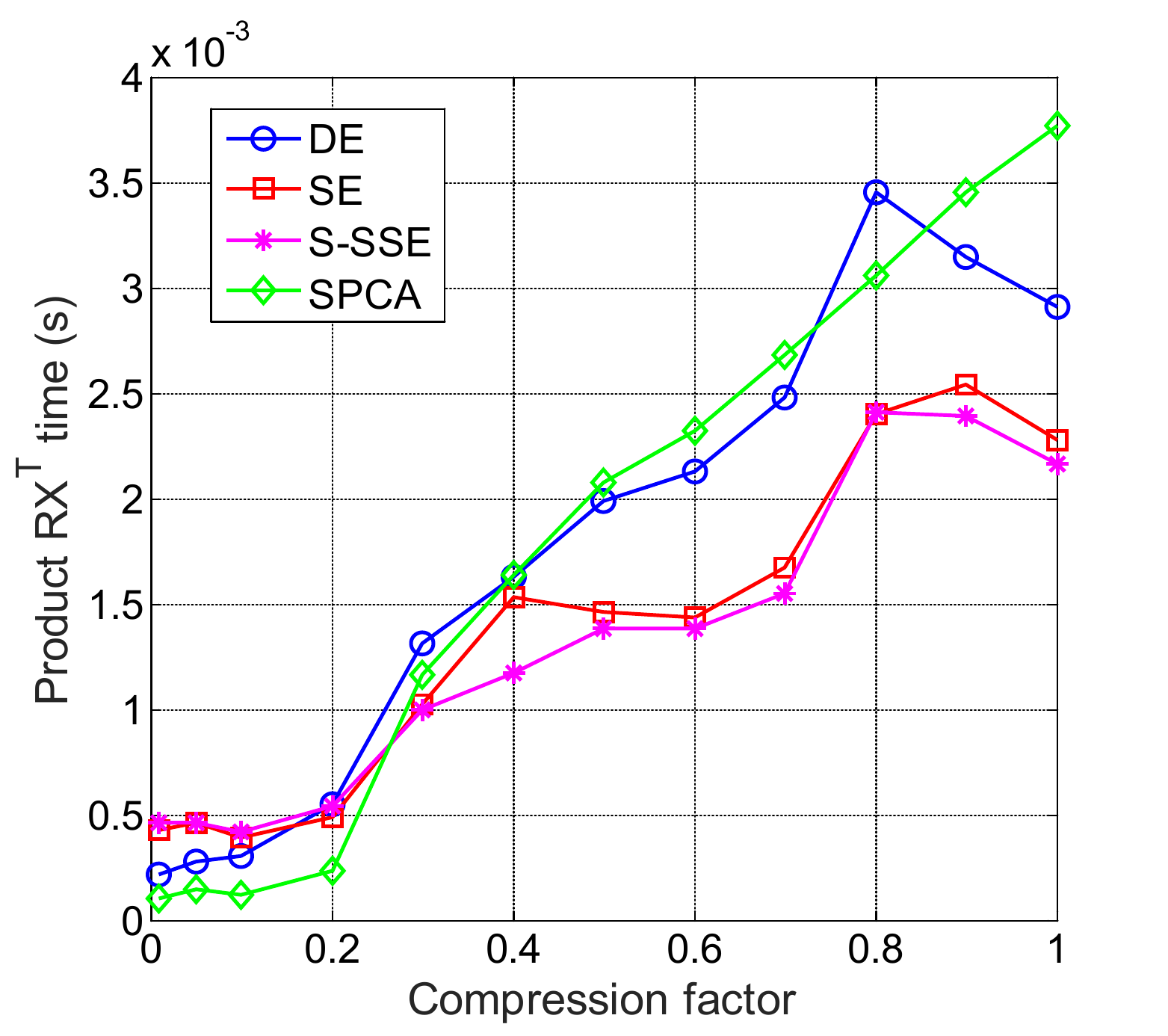}}
 \subfigure[USPS]{
    \label{fig:usps_XR} 
    \includegraphics[width=0.31\textwidth]{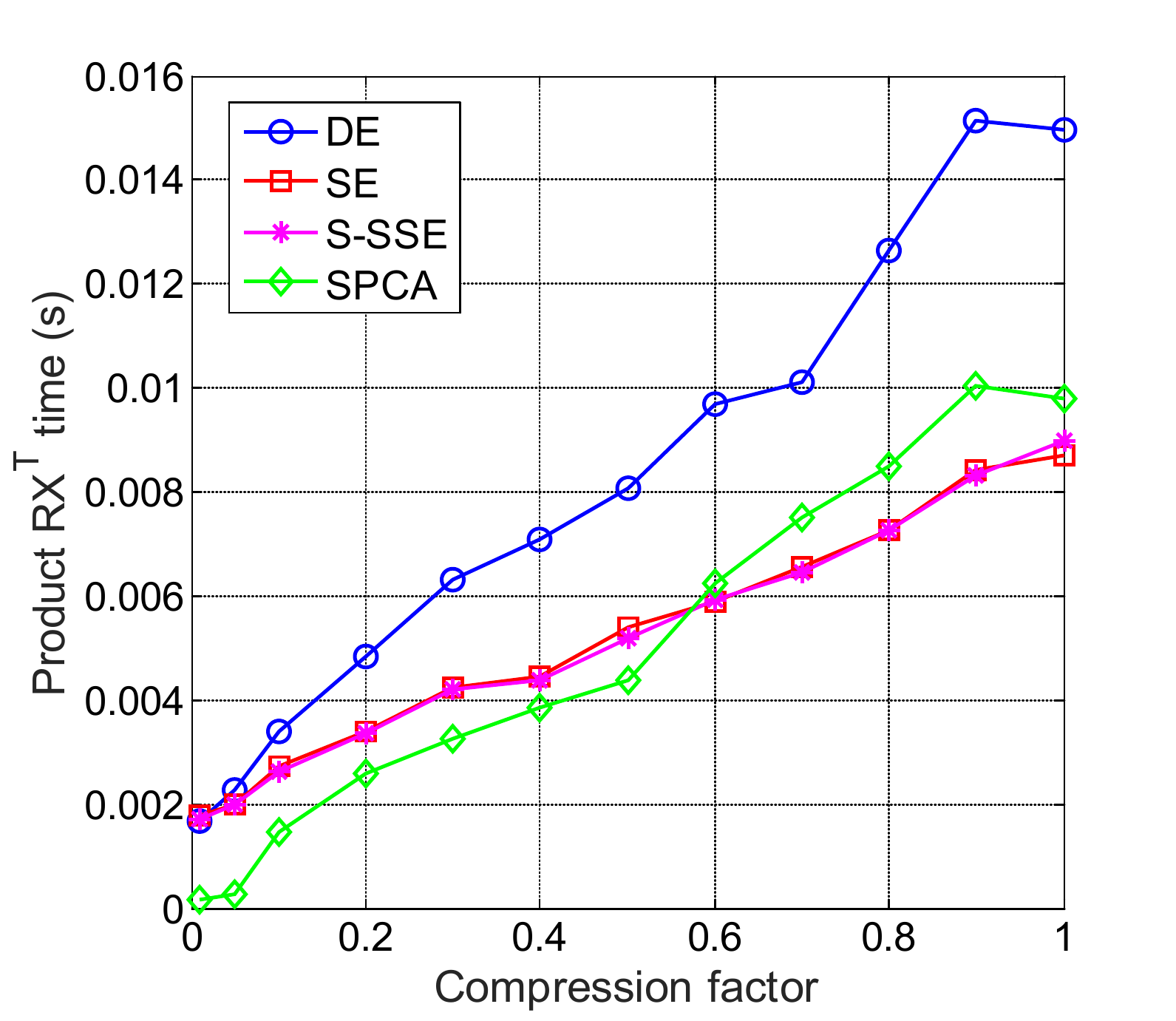}}
 \subfigure[MADELON]{
    \label{fig:madelon_XR} 
    \includegraphics[width=0.31\textwidth]{madelon_scale_mean0var1_time_R_3alg_spca}}
 \vfill
 \subfigure[MNIST]{
    \label{fig:mnist_XR} 
    \includegraphics[width=0.31\textwidth]{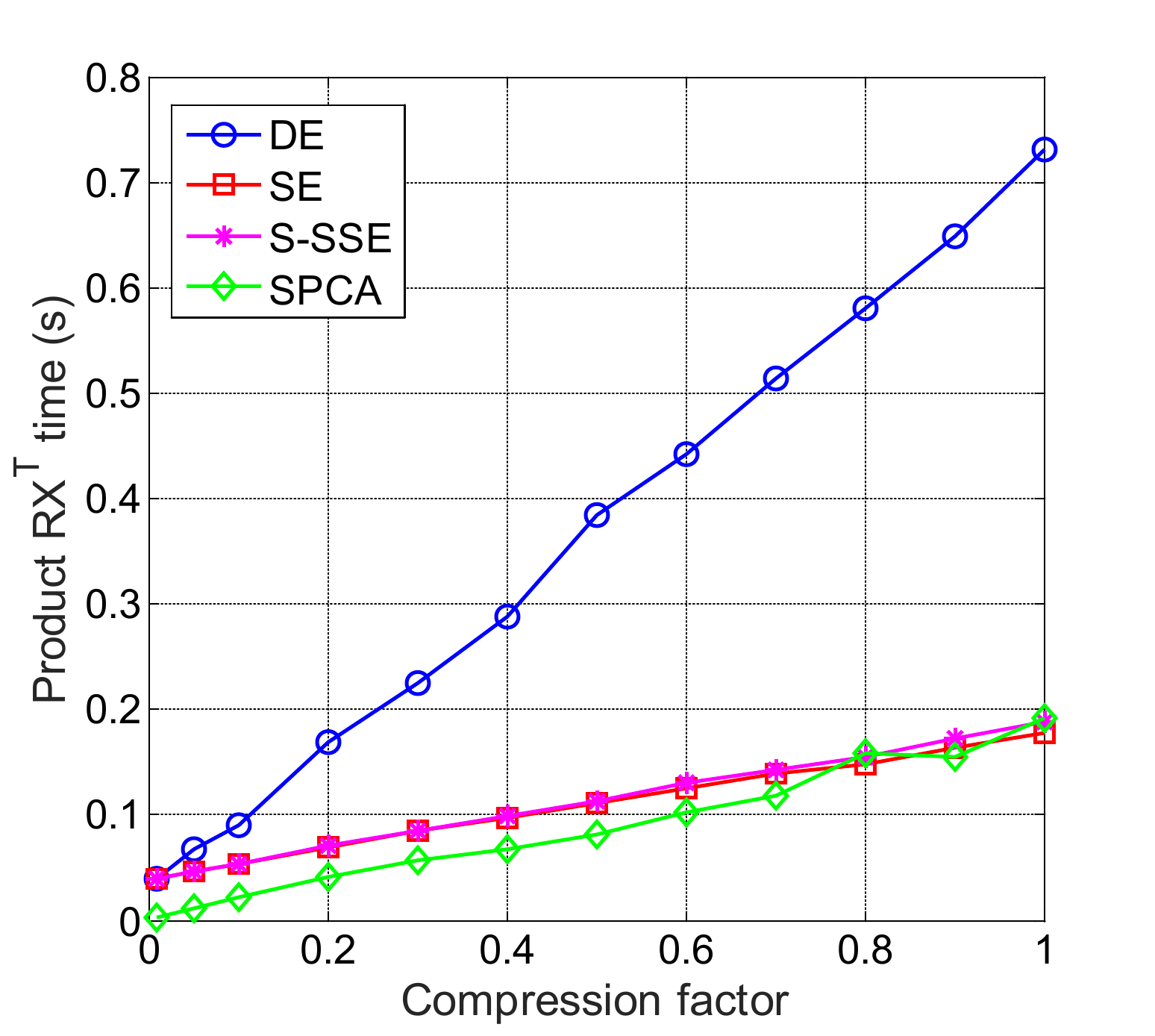}}
 \subfigure[GISETTE]{
    \label{fig:gisette_XR} 
    \includegraphics[width=0.31\textwidth]{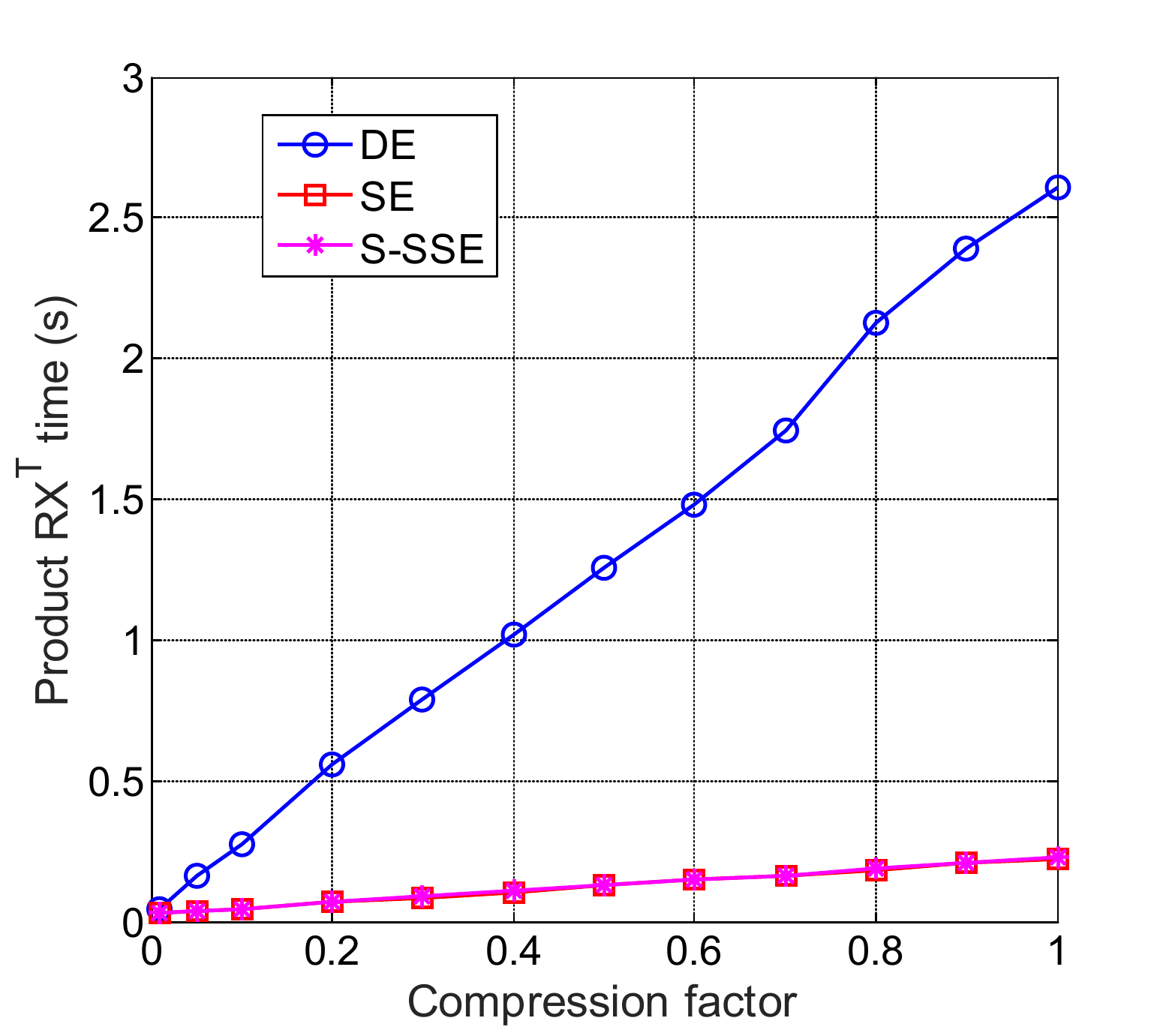}}
  \subfigure[SECTOR]{
    \label{fig:sector} 
    \includegraphics[width=0.31\textwidth]{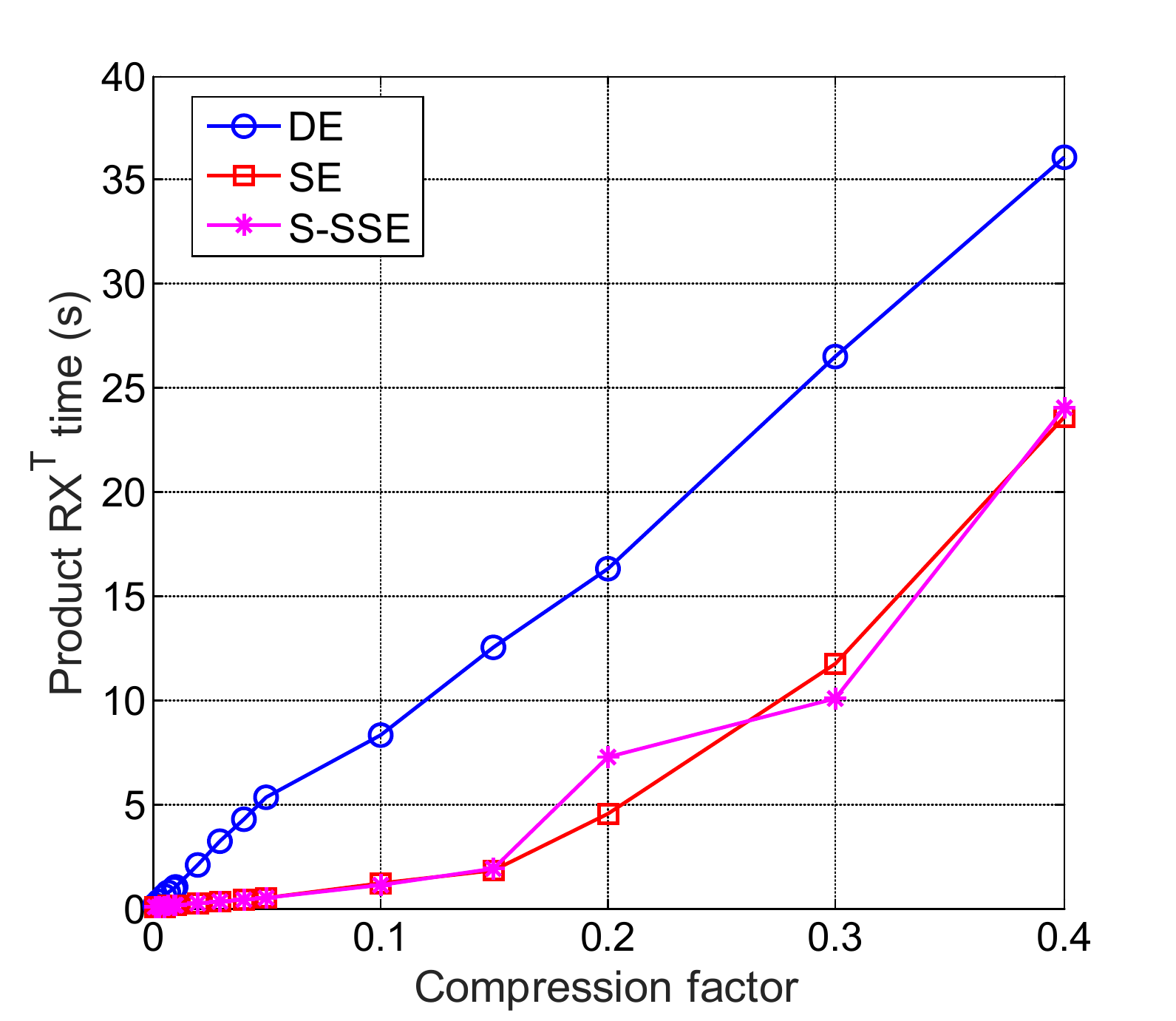}}
   \caption{The time of computing the multiplication $RX^\top$ for various methods.}
 \label{fig:time_XR}
 \end{figure}

From Figs. \ref{fig:acc} - \ref{fig:time_XR}, we can draw the following conclusions:
\begin{itemize}
   \item Fig. \ref{fig:acc} indicates that the S-SSE has superior performance comparing with other RP based methods in terms of accuracy. This verifies our theoretical results. For high-dimension datasets, such as GISETTE, USPS and MNIST, the compression factor can be set as a very small number (0.2, 0.4, 0.3 for these three datases respectively) to obtain satisfactory performance. The SPCA performs well on one dataset DNA, but on the other datasets, its accuracy is lower than other comparison methods. Moreover, with the increasing of extracted dimensions, the accuracy of SPCA may decline because the extracted features may contain noise.
  \item When compression factor is set as $1$, our algorithm has the same accuracy with standard $k$-means, while the DE and the SE have lower accuracy on some datasets, such as DNA and MADELON. That is because non-zero entries are distributed uniformly in our matrix. When $compression~factor=1$, the S-SSE matrix equals to a identity matrix whose columns are permuted, thus features are unchanged after feature extraction. Whereas, even when $compression~factor=1$, feature extraction by the DE and the SE are still the linear combination of original features rather than the original features themselves, which leads to lower accuracy.
  \item With regard to running time, the S-SSE and the SE are very similar to each other for constructing $R$ and for computing product $RX^\top$ on all datasets, which means that our method does not increase running time comparing with the SE, while the performance is improved. The DE is the slower method. That is because the DE matrix is not a sparse matrix, generating it and multiplying it with dataset matrix $X$ are time consuming. The SPCA is the slowest method to construct $R$, because SPCA needs to solve a optimization problem to obtain $R$, which is not easy and the computation is extraordinarily time consumption.
\end{itemize}

\section{Conclusion}
High dimensional data has provided a considerable challenge in designing machine learning algorithm. To address this obstacle, researchers apply dimensionality reduction algorithms first instead of directly working with high dimensional data. Random projection is more efficient than low rank based approaches, therefore it attracts a lot of researchers to study.
In this study, we design a stable sparse subspace embedding algorithm for dimensionality reduction. It overcomes the disadvantages of the state-of-art sparse embedding methods, such as the instability of matrix, the uneven distribution of nonzeros among columns in matrix. It is proved that the proposed method is stabler than the existing method, and it can preserve $(1+\epsilon)$-approximation after dimensionality reduction. The superior performance of our method are attributed to the uniform distribution of nonzeros in the matrix. The experimental results verify our theoretical analysis and show that compared with other dimensionality reduction methods, the new algorithm is stabler, can better maintain Euclidean distance between points, and can obtain better performance in machine learning algorithm.
We conclude this paper with two open questions. Is our stable idea effective for other RP approaches? Does our algorithm perform well on other machine learning algorithms besides $k$-mean clustering?
\section*{Acknowledgements}
This work is supported by the National Natural Science Foundation of China (NNSFC) [No. 61772020].
\bibliography{dim_reduce}

\begin{thebibliography}{10}
\expandafter\ifx\csname url\endcsname\relax
  \def\url#1{\texttt{#1}}\fi
\expandafter\ifx\csname urlprefix\endcsname\relax\def\urlprefix{URL }\fi
\expandafter\ifx\csname href\endcsname\relax
  \def\href#1#2{#2} \def\path#1{#1}\fi

\bibitem{Boutsidis2010}
C.~Boutsidis, A.~Zouzias, P.~Drineas, Random projections for $k$-means
  clustering, in: Advances in Neural Information Processing Systems, 2010, pp.
  298--306.

\bibitem{Sinha2018icml}
K.~Sinha, $k$-means clustering using random matrix sparsification, in:
  International Conference on Machine Learning, 2018, pp. 4691--4699.

\bibitem{Cai2017knowledge}
W.~Cai, A dimension reduction algorithm preserving both global and local
  clustering structure, Knowledge-Based Systems 118 (2017) 191--203.

\bibitem{shi2012margin}
Q.~Shi, C.~Shen, R.~Hill, A.~Van Den~Hengel, Is margin preserved after random
  projection?, in: Proceedings of the 29th International Coference on
  International Conference on Machine Learning, 2012, pp. 643--650.

\bibitem{zhang2013recovering}
L.~Zhang, M.~Mahdavi, R.~Jin, T.~Yang, S.~Zhu, Recovering the optimal solution
  by dual random projection, in: Conference on Learning Theory, 2013, pp.
  135--157.

\bibitem{kumar2008randomized}
K.~Kumar, C.~Bhattacharya, R.~Hariharan, A randomized algorithm for large scale
  support vector learning, in: Advances in Neural Information Processing
  Systems, 2008, pp. 793--800.

\bibitem{paul2014random}
S.~Paul, C.~Boutsidis, M.~Magdon-Ismail, P.~Drineas, Random projections for
  linear support vector machines, ACM Transactions on Knowledge Discovery from
  Data (TKDD) 8~(4) (2014) 1--25.

\bibitem{deegalla2006reducing}
S.~Deegalla, H.~Bostrom, Reducing high-dimensional data by principal component
  analysis vs. random projection for nearest neighbor classification, in: 2006
  5th International Conference on Machine Learning and Applications (ICMLA'06),
  IEEE, 2006, pp. 245--250.

\bibitem{Kenneth2017}
K.~L. Clarkson, D.~P. Woodruff, Low-rank approximation and regression in input
  sparsity time, Journal of the ACM (JACM) 63~(6) (2017) 54.

\bibitem{Michael2015}
M.~B. Cohen, S.~Elder, C.~Musco, C.~Musco, M.~Persu, Dimensionality reduction
  for $k$-means clustering and low rank approximation, in: Proceedings of the
  forty-seventh Annual ACM Symposium on Theory of Computing. ACM, 2015.

\bibitem{Zou2006spca}
H.~Zou, T.~Hastie, R.~Tibshirani,
  \href{https://doi.org/10.1198/106186006X113430}{Sparse principal component
  analysis}, Journal of Computational and Graphical Statistics 15~(2) (2006)
  265--286.
\newblock \href {http://arxiv.org/abs/https://doi.org/10.1198/106186006X113430}
  {\path{arXiv:https://doi.org/10.1198/106186006X113430}}, \href
  {https://doi.org/10.1198/106186006X113430}
  {\path{doi:10.1198/106186006X113430}}.
\newline\urlprefix\url{https://doi.org/10.1198/106186006X113430}

\bibitem{Shen2008spca}
H.~Shen, J.~Z. Huang, Sparse principal component analysis via regularized low
  rank matrix approximation, Journal of Multivariate Analysis 99~(6) (2008)
  1015--1034.

\bibitem{leng2009aspca}
C.~Leng, H.~Wang, On general adaptive sparse principal component analysis,
  Journal of Computational \& Graphical Statistics 18~(1) (2009) 201--215.

\bibitem{Farhad2017}
F.~Pourkamali-Anaraki, S.~Becker, Preconditioned data sparsification for big
  data with applications to {PCA} and $k$-means, IEEE Transactions on
  Information Theory 63~(5) (2017) 2954--2974.

\bibitem{bingham2001random}
E.~Bingham, H.~Mannila, Random projection in dimensionality reduction:
  applications to image and text data, in: Proceedings of the seventh ACM
  SIGKDD International Conference on Knowledge Discovery and Data Mining, ACM,
  2001, pp. 245--250.

\bibitem{goel2005face}
N.~Goel, G.~Bebis, A.~Nefian, Face recognition experiments with random
  projection, in: Biometric Technology for Human Identification II, Vol. 5779,
  International Society for Optics and Photonics, 2005, pp. 426--438.

\bibitem{liu2006random}
K.~Liu, H.~Kargupta, J.~Ryan, Random projection-based multiplicative data
  perturbation for privacy preserving distributed data mining, IEEE
  Transactions on Knowledge and Data Engineering 18~(1) (2006) 92--106.

\bibitem{arriaga1999algorithmic}
R.~I. Arriaga, S.~Vempala, An algorithmic theory of learning: Robust concepts
  and random projection, in: 40th Annual Symposium on Foundations of Computer
  Science, IEEE, 1999, pp. 616--623.

\bibitem{Achlioptas2001}
D.~Achlioptas, Database-friendly random projections: {J}ohnson-{L}indenstrauss
  with binary coins, Journal of Computer and System Sciences 66~(3) (2001)
  671--687.

\bibitem{dasgupta1999learning}
S.~Dasgupta, Learning mixtures of gaussians, in: 40th Annual Symposium on
  Foundations of Computer Science, IEEE, 1999, pp. 634--644.

\bibitem{Daniel2014}
D.~M. Kane, J.~Nelson, Sparser {J}ohnson-{L}indenstrauss transforms, Journal of
  the ACM (JACM) 61~(1) (2014) 4:1--23.

\bibitem{Li2006}
P.~Li, T.~J. Hastie, K.~W. Church, Very sparse random projections, in:
  Proceedings of the 12th ACM SIGKDD International Conference on Knowledge
  Discovery and Data Mining. ACM, 2006.

\bibitem{liu2017sparse}
W.~Liu, X.~Shen, I.~Tsang, Sparse embedded $k$-means clustering, in: Advances
  in Neural Information Processing Systems, 2017, pp. 3321--3329.

\bibitem{Johnson1984}
W.~B. Johnson, J.~Lindenstrauss, Extensions of {L}ipschitz mappings into a
  {H}ilbert space, in: Contemporary Mathematics, Vol.~26, 1984, pp. 189--206.

\bibitem{Diakonikolas2010}
I.~Diakonikolas, D.~M. Kane, J.~Nelson, Bounded independence fools degree-2
  threshold functions, in: 51st Annual IEEE Symposium on Foundations of
  Computer Science (FOCS), 2010, pp. 11--20.

\bibitem{fahad2014survey}
A.~Fahad, N.~Alshatri, Z.~Tari, A.~Alamri, I.~Khalil, A.~Y. Zomaya, S.~Foufou,
  A.~Bouras, A survey of clustering algorithms for big data: Taxonomy and
  empirical analysis, IEEE transactions on emerging topics in computing 2~(3)
  (2014) 267--279.

\end{thebibliography}
\end{document}